\documentclass[]{article}
\usepackage{proceed2e}

\usepackage{times}

\usepackage{epsfig}
\usepackage{amsmath,amsthm,amssymb}
\usepackage{url}
\usepackage{rotating}
\usepackage{soul}

\newcommand{\pa}[0]{\textrm{pa}}

\newcommand{\dkls}[3]{\mathbb{D}_{KL}^{#1}[#2 \, \|\, #3]}
\newcommand{\ddd}[3]{\mathbb{D}(#2 \, \|\, #3)}

\newcommand\cut[1]{}

\newcommand{\elbofinal}{\underline{\mathcal{L}}}

\newcommand{\tm}{\widetilde{m}}
\newcommand{\tv}{\widetilde{v}}

\newcommand{\tvm}{\widetilde{\vm}}

\newcommand{\tvgamma}{\widetilde{\vgamma}}

\newcommand{\tp}{\tilde{p}}

\usepackage{algorithm}

\newcommand{\squishlist}{
   \begin{list}{$\bullet$}
    { \setlength{\itemsep}{0pt}      \setlength{\parsep}{3pt}
      \setlength{\topsep}{3pt}       \setlength{\partopsep}{0pt}
      \setlength{\leftmargin}{1.5em} \setlength{\labelwidth}{1em}
      \setlength{\labelsep}{0.5em} } }

\newcommand{\squishlisttwo}{
   \begin{list}{$\bullet$}
    { \setlength{\itemsep}{0pt}    \setlength{\parsep}{0pt}
      \setlength{\topsep}{0pt}     \setlength{\partopsep}{0pt}
      \setlength{\leftmargin}{2em} \setlength{\labelwidth}{1.5em}
      \setlength{\labelsep}{0.5em} } }

\newcommand{\squishend}{
    \end{list}  }
 
{}
\newtheorem{thm}{Theorem}{}
\newtheorem{prop}{Proposition}{}
{}
\newtheorem{lemma}{Lemma}{}

\newcommand{\half}{\mbox{$\frac{1}{2}$}}

\newcommand{\rnd}[1]{\left(#1\right)}
\newcommand{\sqr}[1]{\left[#1\right]}
\newcommand{\crl}[1]{\left\{#1\right\}}
\newcommand{\myexpect}{\mathbb{E}}

\newcommand{\gauss}{\mbox{${\cal N}$}}

\newcommand{\myvec}[1]{\mbox{$\mathbf{#1}$}}
\newcommand{\myvecsym}[1]{\mbox{$\boldsymbol{#1}$}}

\newcommand{\vone}{\mbox{$\myvecsym{1}$}}

\newcommand{\vbeta}{\mbox{$\myvecsym{\beta}$}}

\newcommand{\vdelta}{\mbox{$\myvecsym{\delta}$}}

\newcommand{\veta}{\mbox{$\myvecsym{\eta}$}}
\newcommand{\vgamma}{\mbox{$\myvecsym{\gamma}$}}
\newcommand{\vmu}{\mbox{$\myvecsym{\mu}$}}
\newcommand{\vlambda}{\mbox{$\myvecsym{\lambda}$}}

\newcommand{\vtheta}{\mbox{$\myvecsym{\theta}$}}

\newcommand{\vSigma}{\mbox{$\myvecsym{\Sigma}$}}

\newcommand{\vxi}{\mbox{$\myvecsym{\xi}$}}

\newcommand{\vg}{\mbox{$\myvec{g}$}}

\newcommand{\vm}{\mbox{$\myvec{m}$}}

\newcommand{\vs}{\mbox{$\myvec{s}$}}

\newcommand{\vx}{\mbox{$\myvec{x}$}}

\newcommand{\vy}{\mbox{$\myvec{y}$}}

\newcommand{\vz}{\mbox{$\myvec{z}$}}

\newcommand{\vA}{\mbox{$\myvec{A}$}}
\newcommand{\vB}{\mbox{$\myvec{B}$}}

\newcommand{\vG}{\mbox{$\myvec{G}$}}

\newcommand{\vI}{\mbox{$\myvec{I}$}}

\newcommand{\vK}{\mbox{$\myvec{K}$}}
\newcommand{\vL}{\mbox{$\myvec{L}$}}

\newcommand{\vT}{\mbox{$\myvec{T}$}}

\newcommand{\vV}{\mbox{$\myvec{V}$}}

\newcommand{\vX}{\mbox{$\myvec{X}$}}

\newcommand{\diag}{\mbox{$\mbox{diag}$}}

\newcommand{\trace}{\mbox{Tr}}

\newcommand{\calD}{\mbox{${\cal D}$}}

\newcommand{\data}{\calD}

\newcommand{\be}{\begin{equation}}
\newcommand{\ee}{\end{equation}}
\newcommand{\bea}{\begin{eqnarray}}
\newcommand{\eea}{\end{eqnarray}}
\newcommand{\beaa}{\begin{eqnarray*}}
\newcommand{\eeaa}{\end{eqnarray*}}

\usepackage{natbib}
\usepackage{subfigure}
\usepackage{graphicx}
\usepackage{verbatim}

\def\norm#1{\|#1\|}
\newcommand{\argmin}[1]{\mathop{\hbox{argmin}}_{#1}}

\title{Faster Stochastic Variational Inference using\\ Proximal-Gradient Methods with General Divergence Functions}

\author{
Mohammad Emtiyaz Khan \\
Ecole Polytechnique F\'{e}d\'{e}rale de Lausanne\\
Lausanne Switzerland\\
\texttt{emtiyaz@gmail.com} \\
\And
Reza Babanezhad\\
University of British Columbia\\
Vancouver, Canada\\
\texttt{babanezhad@gmail.com} \\
\AND
Wu Lin\\
University of Waterloo\\
Waterloo, Canada\\
\texttt{wu.lin@uwaterloo.ca} \\
\And
Mark Schmidt\\
University of British Columbia\\
Vancouver, Canada\\
\texttt{schmidtm@cs.ubc.ca} \\
\And
Masashi Sugiyama\\
University of Tokyo\\
Tokyo, Japan\\
\texttt{sugi@k.u-tokyo.ac.jp} \\
}

\begin{document}

\maketitle

\begin{abstract} 
Several recent works have explored stochastic gradient methods for variational inference that exploit the geometry of the variational-parameter space. However, the theoretical properties of these methods are not well-understood and these methods typically only apply to conditionally-conjugate models. We present a new stochastic method for variational inference which exploits the geometry of the variational-parameter space and also yields simple closed-form updates even for non-conjugate models. We also give a convergence-rate analysis of our method and many other previous methods which exploit the geometry of the space. Our analysis generalizes existing convergence results for stochastic mirror-descent on non-convex objectives by using a more general class of divergence functions. Beyond giving a theoretical justification for a variety of recent methods, our experiments show that new algorithms derived in this framework lead to state of the art results on a variety of problems. Further, due to its generality, we expect that our theoretical analysis could also apply to other applications.
\end{abstract} 

\section{INTRODUCTION}

Variational inference methods are one of the most widely-used computational tools to deal with the intractability of Bayesian inference, while stochastic gradient (SG) methods are one of the most widely-used tools for solving optimization problems on huge datasets. The last three years have seen an explosion of work exploring SG methods for variational inference~\citep{hoffman2013stochastic,salimans2013fixed, ranganath2013black, titsias2014doubly, mnih2014neural, kucukelbir2014fully}. In many settings, these methods can yield simple updates and scale to huge datasets.

A challenge that has been addressed in many of the recent works on this topic is that the ``black-box" SG method ignores the geometry of the variational-parameter space. This has lead to methods like the stochastic variational inference (SVI) method of~\citet{hoffman2013stochastic}, that uses \emph{natural gradients} to exploit the geometry. This leads to better performance in practice, but this approach only applies to conditionally-conjugate models. In addition, it is not clear how using natural gradients for variational inference affects the theoretical convergence rate of SG methods.

In this work we consider a general framework that (i) can be stochastic to allow huge datasets, (ii) can exploit the geometry of the variational-parameter space to improve performance, and (iii) can yield a closed-form update even for non-conjugate models. The new framework can be viewed as a stochastic generalization of the proximal-gradient method of~\citet{Khan15nips}, which splits the objective into conjugate and non-conjugate terms. By linearizing the non-conjugate terms, this previous method as well as our new method yield simple closed-form proximal-gradient updates even for non-conjugate models. 

While proximal-gradient methods have been well-studied in the optimization community~\citep{beck2009fast}, like SVI there is nothing known about the convergence rate of the method of~\citet{Khan15nips} because it uses ``divergence" functions which do not satisfy standard assumptions.
Our second contribution is to \emph{analyze the convergence rate} of the proposed method. In particular, we generalize an existing result on the convergence rate of stochastic mirror descent in non-convex settings~\citep{ghadimi2014mini} to allow a  general class of divergence functions that includes the cases above (in both deterministic and stochastic settings). While it has been observed empirically that including an appropriate divergence function enables larger steps than basic SG methods, this work gives the first theoretical result justifying the use of these more-general divergence functions. It in particular reveals how different factors affect the convergence rate such as the Lipschitz-continuity of the lower bound, the information geometry of the divergence functions, and the variance of the stochastic approximation. Our results also suggest conditions under which the proximal-gradient steps of~\citet{Khan15nips} can make more progress than (non-split) gradient steps, and sheds light on the choice of step-size for these methods. 
Our experimental results indicate that the new method leads to improvements in performance on a variety of problems, and we note that the algorithm and theory might be useful beyond the variational inference scenarios we have considered in this work.


\section{VARIATIONAL INFERENCE} \label{sec:variational_inference}

Consider a general latent variable model where we have a data vector $\vy$ of length $N$ and a latent vector $\vz$ of length $D$. 
In Bayesian inference, we are interested in computing the marginal likelihood $p(\vy)$, which can be written as the integral of the joint distribution $p(\vy,\vz)$ over all values of $\vz$.
This integral is often intractable, and in variational inference we typically approximate it with the evidence lower-bound optimization (ELBO) approximation $\elbofinal$. This approximation introduces a distribution $q(\vz|\vlambda)$ and chooses the variational parameters $\vlambda$ to maximize the following lower bound on the marginal likelihood:
\begin{equation}
\begin{aligned}
 &  \log p(\vy) = \log\int q(\vz|\vlambda)\frac{p(\vy,\vz)}{q(\vz|\vlambda)}\, d\vz, \\
 &  \ge \max_{\boldsymbol{\lambda} \in \mathcal{S}} \elbofinal(\vlambda) := \myexpect_{q(\mathbf{z}|\boldsymbol{\lambda})} \sqr{ \log \frac{p(\vy,\vz)}{q(\vz|\vlambda)} } . 
\end{aligned}
\label{eq:LB} 
\end{equation}
The  inequality follows from concavity of the logarithm function. The set $\mathcal{S}$ is the set of valid parameters $\vlambda$.

To optimize $\vlambda$, one of the seemingly-simplest approaches is gradient descent:
$\vlambda_{k+1} = \vlambda_k + \beta_k \nabla \elbofinal(\vlambda_k)$, 
which can be viewed as optimizing a quadratic approximation of $\elbofinal$,
\begin{align}
\vlambda_{k+1} &= \argmin{\boldsymbol{\lambda}\in\mathcal{S}} \sqr{ -\vlambda^T \nabla \elbofinal(\vlambda_k) + \frac{1}{2\beta_k} \|\vlambda - \vlambda_k \|^2_2}. \label{eq:equivalent_grad_descent}
\end{align}
While we can often choose the family $q$ so that it has convenient computational properties, it might be impractical to apply gradient descent in this context when we have a very large dataset or when some terms in the lower bound are intractable.
Recently, SG methods have been proposed to deal with these issues~\citep{ranganath2013black, titsias2014doubly}: they allow large datasets by using random subsets (mini-batches) and can approximate intractable integrals using Monte Carlo methods that draw samples from $q(\vz|\vlambda)$.

A second drawback of applying gradient descent to variational inference is that it uses the Euclidean distance and thus ignores the \emph{geometry of the variational-parameter space}, which often results in slow convergence. Intuitively, \eqref{eq:equivalent_grad_descent} implies that we should move in the direction of the gradient, but not move $\vlambda_{k+1}$ too far away from $\vlambda_k$ in terms of the Euclidean distance. However, the Euclidean distance is not appropriate for variational inference because $\vlambda$ is the parameter vector of a distribution; the Euclidean distance is often a poor measure of dissimilarity between distributions. The following example from \citet{hoffman2013stochastic} illustrates this point: the two normal distributions $\gauss(0,10000)$ and $\gauss(10,10000)$ are almost indistinguishable, yet the Euclidean distance between their parameter vectors is 10, whereas the distributions $\gauss(0,0.01)$ and $\gauss(0.1,0.01)$ barely overlap, but their Euclidean distance between parameters is only $0.1$.


{\bf Natural-Gradient Methods:}
The canonical way to address the problem above is by replacing the Euclidean distance in \eqref{eq:equivalent_grad_descent} with another divergence function. For example, the \emph{natural gradient} method defines the iteration by using the symmetric Kullback-Leibler (KL) divergence~\citep{hoffman2013stochastic, pascanu2013revisiting, amari1998natural},
\begin{equation}
\label{eq:naturalgradient} 
\begin{aligned}
&\vlambda_{k+1} =\\
&\argmin{\boldsymbol{\lambda}\in\mathcal{S}} \sqr{ -\vlambda^T \nabla \elbofinal(\vlambda_k) + \frac{1}{\beta_k}\dkls{sym}{q(\vz|\vlambda)}{q(\vz|\vlambda_k)}}.
\end{aligned}
\end{equation}
This leads to the update
\begin{align}
\vlambda_{k+1} = \vlambda_{k} + \beta_k \sqr{\nabla^2 \vG(\vlambda_k)}^{-1} \nabla \elbofinal(\vlambda_{k}), \label{eq:natural_grad}
\end{align}
where $\vG(\vlambda)$ is the Fisher information-matrix, 
\begin{align*}
\vG(\vlambda) := \myexpect_{q(\mathbf{z}|\boldsymbol{\lambda})} \crl{\sqr{\nabla \log q(\vz|\vlambda)}\sqr{\nabla \log q(\vz|\vlambda)}^T} . \nonumber
\end{align*}
\cite{hoffman2013stochastic} show that the natural-gradient update can be computationally simpler than gradient descent for conditionally-conjugate exponential family models. In this family, we assume that the distribution of $\vz$ factorizes as $\prod_i p(\vz^i|\pa^i)$ where $\vz^i$ are disjoint subsets of $\vz$ and $\pa^i$ are the parents of the $\vz^i$ in a directed acyclic graph. This family also assumes that each conditional distribution is in the exponential family,
\[p(\vz^i|\pa^i) := h^i(\vz^i) \exp\sqr{ [\veta^i(\pa^i)]^T \vT^i(\vz^i) - A^i(\veta^i)},\]  
where $\veta^i$ are the natural parameters, $\vT^i(\vz^i)$ are the sufficient statistics, $A^i(\veta^i)$ is the partition function, and $h^i(\vz^i)$ is the base measure.
\cite{hoffman2013stochastic} consider a mean-field approximation $q(\vz|\vlambda) = \prod_i q^i(\vz^i|\vlambda^i)$ where each $q^i$ belongs to the same exponential-family distribution as the joint distribution,
\begin{align*}
q^i(\vz^i) := h^i(\vz^i) \exp\sqr{ (\vlambda^i)^T \vT^i(\vz^i) - A^i(\vlambda^i)} .   
\end{align*}
The parameters of this distribution are denoted by $\vlambda^i$ to differentiate them from the joint-distribution parameters $\veta^i$.

As shown by \cite{hoffman2013stochastic}, the Fisher matrix for this problem is equal to $\nabla^2 A^i(\vlambda^i)$ and the gradient of the lower bound with respect to $\vlambda^i$ is equal to $\nabla^2 A^i(\vlambda^i) (\vlambda^i - \vlambda^i_*)$ where $\vlambda^i_*$ are the mean-field parameters~\citep[see][]{paquetconvergence}. Therefore, when computing the natural-gradient, the $\nabla^2 A^i(\vlambda^i)$ terms cancel out and the natural-gradient is simply $\vlambda^i - \vlambda^i_*$ which is much easier to compute than the actual gradient. Unfortunately, for non-conjugate models this cancellation does not happen and the simplicity of the update is lost. The Riemannian conjugate-gradient method of~\citet{Honkela:11} has similar issues, in that computing $\nabla^2 A(\vlambda)$ is typically very costly.

{\bf KL-Divergence Based Methods:}
Rather than using the symmetric-KL,~\citet{theis2015trust} consider using the KL divergence $\dkls{}{q(\vz|\vlambda)}{q(\vz|\vlambda_k)}$ within a stochastic proximal-point method:
\begin{equation}
\begin{aligned}
\vlambda_{k+1} &= \argmin{\boldsymbol{\lambda}\in\mathcal{S}} \sqr{ - \elbofinal(\vlambda) + \frac{1}{\beta_k}\dkls{}{q(\vz|\vlambda)}{q(\vz|\vlambda_k)} }. \label{eq:theis}
\end{aligned}
\end{equation}
This method yields better convergence properties, but requires numerical optimization to implement the update even for conditionally-conjugate models. \citet{Khan15nips} considers a deterministic proximal-gradient variant of this method by splitting the lower bound into $-\elbofinal := f + h$, where $f$ contains all the ``easy" terms and $h$ contains all the ``difficult" terms. By linearizing the ``difficult" terms, this leads to a closed-form update even for non-conjugate models. The update is given by:
\begin{equation}
\begin{aligned}
\vlambda_{k+1} &= \argmin{\boldsymbol{\lambda}\in\mathcal{S}}\left[ \vlambda^T[ \nabla f(\vlambda_k)] + h(\vlambda) \right. \\
&\quad\quad\quad\quad\quad\quad \left. + \frac{1}{\beta_k}\dkls{}{q(\vz|\vlambda)}{q(\vz|\vlambda_k)}\right].\label{eq:khan2015}
\end{aligned}
\end{equation}
However, this method requires the exact gradients which is usually not  feasible for large dataset and/or complex models. 

{\bf Mirror Descent Methods:}
In the optimization literature, \emph{mirror descent} (and stochastic mirror descent) algorithms are a generalization of~\eqref{eq:equivalent_grad_descent} where the squared-Euclidean distance can be replaced by any Bregman divergence $\mathbb{D}_F(\vlambda \| \vlambda_k)$ generated from a strongly-convex function $F(\vlambda)$~\citep{beck2003mirror},
\begin{align}
\vlambda_{k+1} &= \argmin{\boldsymbol{\lambda}\in\mathcal{S}} \left\{-\vlambda^T \nabla \elbofinal(\vlambda_k) + \frac{1}{\beta_k} \mathbb{D}_{F}(\vlambda\|\vlambda_k) \right\}. \label{eq:mirror_descent}
\end{align}
The convergence rate of mirror descent algorithm has been analyzed in convex~\citep{duchi2010composite} and more recently in non-convex~\citep{ghadimi2014mini} settings. However, mirror descent does not cover the cases described above in \eqref{eq:theis} and \eqref{eq:khan2015} when a KL divergence between two exponential-family distributions is used with $\vlambda$ as the natural-parameter. For such cases, the Bregman divergence corresponds to a KL divergence with swapped parameters~\citep[see][Equation~29]{nielsen2009statistical},
\begin{align}
\mathbb{D}_{A}(\vlambda\|\vlambda_k) &:= A(\vlambda) - A(\vlambda_k) - [\bigtriangledown A(\vlambda_k)]^T (\vlambda - \vlambda_k) \nonumber \\
&= \dkls{}{q(\vz|\vlambda_k)}{q(\vz|\vlambda)}. \label{eq:breg}
\end{align}
where $A(\vlambda)$ is the partition function of $q$. Because \eqref{eq:theis} and \eqref{eq:khan2015} both use a KL divergence where the second argument is fixed to $\vlambda_k$, instead of the first argument, they are not covered under the mirror-descent framework. 
In addition, even though mirror-descent has been used for variational inference~\citep{ravikumar2010message}, 
Bregman divergences do not yield an efficient update in many scenarios.

%
%

\section{PROXIMAL-GRADIENT SVI} \label{sec:pgsvi}
Our proximal-gradient stochastic variational inference (PG-SVI) method extends \eqref{eq:khan2015} to allow stochastic gradients $\widehat{\nabla} f(\vlambda_k)$ and general divergence functions $\mathbb{D}(\vlambda\|\vlambda_k)$ 
by using the iteration
  \begin{align}
    \vlambda_{k+1} &= \argmin{\boldsymbol{\lambda} \in \mathcal{S}} \crl{ \vlambda^T \sqr{ \widehat{\bigtriangledown} f(\vlambda_k)} + h(\vlambda) + \frac{1}{\beta_k} \mathbb{D}(\vlambda\, \|\, \vlambda_k) }. \label{eq:subproblem}
  \end{align}
This unifies a variety of existing approaches since it allows:
\begin{enumerate}
\setlength\itemsep{.25em}
\item Splitting of $\elbofinal$ into a difficult term $f$ and a simple term $h$, similar to the method of \cite{Khan15nips}.
\item A stochastic approximation $\widehat{\nabla} f$ of the gradient of the difficult term, similar to SG methods.
\item Divergence functions $\mathbb{D}$ that incorporate the geometry of the parameter space, similar to methods discussed in Section \ref{sec:variational_inference} (see~\eqref{eq:naturalgradient}, \eqref{eq:theis}, \eqref{eq:khan2015}, and \eqref{eq:mirror_descent}).
\end{enumerate}
Below, we describe each feature in detail, along with the precise assumptions used in our analysis.

\subsection{SPLITTING}\label{sec:splitting}
Following~\citet{Khan15nips}, we split the lower bound into a sum of a ``difficult" term $f$ and an ``easy" term $h$, enabling a closed-form solution for \eqref{eq:subproblem}. 
Specifically, we split using $p(\vy,\vz)/q(\vz|\vlambda) = c\, \tp_d(\vz|\vlambda) \tp_e(\vz|\vlambda)$, where $\tp_d$ contains all factors that make the optimization difficult, and $\tp_e$ contains the rest (while $c$ is a constant).
By substituting in \eqref{eq:LB}, we get the following split of the lower bound:
\begin{align*}
&\elbofinal(\vlambda) =  \underbrace{\myexpect_q[ \log \tilde{p}_d(\vz|\vlambda)]}_{-f(\boldsymbol{\lambda})} + \underbrace{\myexpect_q[ \log \tilde{p}_e(\vz|\vlambda)]}_{-h(\boldsymbol{\lambda})} + \log c.
\end{align*}
Note that $\tp_d$ and $\tp_e$ need not be probability distributions.

We make the following assumptions about $f$ and $h$:
\begin{description}
\item[(A1)] The function $f$ is differentiable and its gradient is $L-$Lipschitz-continuous, i.e. $\forall \vlambda$ and $\vlambda' \in \mathcal{S}$ we have
\[
\norm{\nabla f(\vlambda) - \nabla f(\vlambda')} \le L\norm{\vlambda -\vlambda'}.
\]
\item[(A2)] The function $h$ can be a general convex function.
\end{description}
These assumptions are very weak. The function $f$ can be non-convex and the Lipschitz-continuity assumption is  typically satisfied in practice (and indeed the analysis can be generalized to only require this assumption on a smaller set containing the iterations). The assumption that $h$ is convex seems strong, but note that we can always take $h = 0$ in the split if the function has no ``nice" convex part. 
Below, we give several illustrative examples of such splits for variational-Gaussian inference with $q(\vz|\vlambda) := \gauss(\vz|\vm,\vV)$, so that $\vlambda = \{\vm,\vV\}$ with $\vm$ being the mean and $\vV$ being the covariance matrix. 

{\bf Gaussian Process (GP) Models:} Consider GP models \citep{Kuss05} for $N$ input-output pairs $\{y_n,\vx_n\}$ indexed by $n$. Let $z_n := f(\vx_n)$ be the latent function drawn from a GP with mean 0 and covariance $\vK$. We use a non-Gaussian likelihood $p(y_n|z_n)$ to model the output. We can then use the following split, where the non-Gaussian terms are in $\tp_d$ and the Gaussian terms are in $\tp_e$:
\begin{align}
\frac{p(\vy,\vz)}{q(\vz|\vlambda)} = \underbrace{\prod_{n=1}^N p(y_n|z_n)}_{\tp_d(\mathbf{z}|\boldsymbol{\lambda})} \underbrace{\frac{\gauss(\vz|0,\vK)}{\gauss(\vz|\vm,\vV)}}_{\tp_e(\mathbf{z}|\boldsymbol{\lambda})}. \label{eq:glm_joint}
\end{align}
The detailed derivation is in the appendix. By substituting in \eqref{eq:LB}, we obtain the lower bound $\elbofinal(\vlambda)$ shown below along with its split:
\begin{align}
\underbrace{ \sum_n \mathbb{E}_{q}[\log p(y_n|z_n)] }_{-f(\boldsymbol{\lambda})} - \underbrace{ \dkls{}{\gauss(\vz|\vm,\vV)}{\gauss(\vz|0,\vK)}}_{h(\boldsymbol{\lambda})}. \label{eq:gp_lb}
\end{align}
A1 is satisfied for common likelihoods, while it is easy to establish that $h$ is convex. We show in Section \ref{sec:example_of_pgsvi} that this split leads to a closed-form update for iteration \eqref{eq:subproblem}.

{\bf Generalized Linear Models (GLMs):} A similar split can be obtained for GLMs \citep{nelder1972generalized}, where the non-conjugate terms are in $\tp_d$ and the rest are in $\tp_e$. Denoting the weights by $\vz$ and assuming a standard Gaussian prior over it, we can use the following split:
\begin{align*}
\frac{p(\vy,\vz)}{q(\vz|\vlambda)} = \underbrace{\prod_{n=1}^N p(y_n|\vx_n^T\vz)}_{\tp_d(\mathbf{z}|\boldsymbol{\lambda})} \underbrace{\frac{\gauss(\vz|0,\vI)}{\gauss(\vz|\vm,\vV)}}_{\tp_e(\mathbf{z}|\boldsymbol{\lambda})}.
\end{align*}
We give further details about the bound for this case in the appendix.

{\bf Correlated Topic Model (CTM):} Given a text document with a vocabulary of $N$ words, denote its word-count vector by $\vy$. Let $K$ be the number of topics and $\vz$ be the vector of topic-proportions. We can then use the following split:
\begin{align*}
\frac{p(\vy,\vz)}{q(\vz|\vlambda)} = \underbrace{\prod_{n=1}^N \sqr{\sum_{k=1}^K \beta_{n,k} \frac{e^{z_k}}{\sum_j e^{z_j}}}^{y_{n}} }_{\tp_d(\mathbf{z}|\boldsymbol{\lambda})} \underbrace{\frac{\gauss(\vz|\vmu,\vSigma)}{\gauss(\vz|\vm,\vV)}}_{\tp_e(\mathbf{z}|\boldsymbol{\lambda})},
\end{align*}
where $\vmu,\vSigma$ are parameters of the Gaussian prior and $\beta_{n,k}$ are parameters of $K$ multinomials. 
We give further details about the bound in the appendix.


\subsection{STOCHASTIC-APPROXIMATION} \label{sec:stoch_approx}
The approach of~\citet{Khan15nips} considers~\eqref{eq:subproblem} in the special case of~\eqref{eq:khan2015} where we use the exact gradient $\nabla f(\vlambda_k)$ in the first term. But in practice this gradient is often difficult to compute. In our framework, we allow a stochastic approximation of $\nabla f(\vlambda)$ which we denote by $\widehat{\nabla} f(\vlambda_k)$. 

As shown in the previous section, $f$ might take a form $f(\vlambda) := \Sigma_{n=1}^N \myexpect_{q} [\tilde{f}_n(\vz)]$ for a set of functions $\tilde{f}_n$ as in the  GP model~\eqref{eq:gp_lb}.
In some situations, $\myexpect_q[\tilde{f}_n(\vz)]$ is computationally expensive or intractable. For example, in GP models the expectation is equal to $\myexpect_q[\log p(y_n|z_n)]$, which is intractable for most non-Gaussian likelihoods. In such cases, we can form a stochastic approximation by using a few samples $\vz^{(s)}$ from $q(\vz|\vlambda)$, as shown below: 
\[
 \nabla  \myexpect_{q} [\tilde{f}_n(\vz)]  \approx
\widehat{\vg}(\vlambda, \vxi_n) := \frac 1 S \sum_{s=1}^S  \tilde{f}_n(\vz^{(s)})\nabla [\log q(\vz^{(s)} | \vlambda)]
\]
where $\vxi_n$ represents the noise in the stochastic approximation $\widehat{\vg}$ and we use the identity $\nabla q(\vz|\vlambda) = q(\vz|\vlambda) \nabla [\log q(\vz|\vlambda)]$ to derive the expression \citep{ranganath2013black}.
We can then form a stochastic-gradient by randomly selecting a mini-batch of $M$ functions $\tilde{f}_{n_i}(\vz)$ and employing the  estimate
\begin{align}
\widehat{\nabla} f(\vlambda) = \frac N M \sum_{i=1}^{M} \widehat{\vg}(\vlambda,\vxi_{n_i}). \label{eq:gradient_approx}
\end{align}
In our analysis we make the following two assumptions regarding the stochastic approximation of the gradient:
\begin{description}
\item[(A3)] The estimate is unbiased: $\myexpect[ \widehat{\vg}(\vlambda,\vxi_{n}) ] = \bigtriangledown f(\vlambda)$.
\item[(A4)] Its variance is upper bounded: $\textrm{Var}[\widehat{\vg}(\vlambda,\vxi_{n})] \le \sigma^2$.
\end{description}
In both the assumptions, the expectation is taken with respect to the noise $\vxi_{n}$. The first assumption is true for the stochastic approximations of \eqref{eq:gradient_approx}. The second assumption is stronger, but only needs to hold for all $\vlambda_k$ so is almost always satisfied in practice. 

\subsection{DIVERGENCE FUNCTIONS}
To incorporate the geometry of $q$ we incorporate a divergence function $\mathbb{D}$ between $\vlambda$ and $\vlambda_k$. The set of divergence functions need to satisfy  two assumptions:
\begin{description}
\item[(A5)] $\mathbb{D}(\vlambda\, \|\, \vlambda') > 0$, for all $\vlambda \ne \vlambda'$.
\item[(A6)] There exist an $\alpha>0$ such that for all $\vlambda, \vlambda'$ generated by \eqref{eq:subproblem} we have:
\begin{align}
(\vlambda - \vlambda')^T \nabla_{\boldsymbol{\lambda}} \mathbb{D}(\vlambda\, \|\, \vlambda') \ge \alpha \|\vlambda - \vlambda'\|^2. \label{eq:A4}
\end{align}
\end{description}
The first assumption is reasonable and is satisfied by typical divergence functions like the squared Euclidean distance and variants of the KL divergence. In the next section we show that, whenever the iteration~\eqref{eq:subproblem} is defined and all $\vlambda_k$ stay within a compact set, the second assumption is satisfied for all divergence functions considered in Section~\ref{sec:variational_inference}.

\section{SPECIAL CASES}

Most methods discussed in Section \ref{sec:variational_inference} are special cases of the proposed iteration~\eqref{eq:subproblem}. 
We obtain gradient descent if $h = 0$, $f = -\elbofinal$ , $\widehat{\nabla} f = \nabla f$, and $\mathbb{D}(\vlambda \| \vlambda_k) = (1/2)\|\vlambda - \vlambda_k\|^2$ (in this case A6 is satisfied with $\alpha = 1$). From here, there are three standard generalizations in the optimization literature: SG methods  do not require that $\widehat{\nabla} f = \nabla f$, proximal-gradient methods do not require that $h = 0$, and mirror descent allows $\mathbb{D}$ to be a different Bregman divergence generated by a strongly-convex function. Our analysis applies to all these variations on existing optimization algorithms because A1 to A5 are standard assumptions~\citep{ghadimi2014mini} and, as we now show, A6 is satisfied for this class of Bregman divergences. In particular, consider the generic Bregman divergence shown in the left side of~\eqref{eq:breg}
for some strongly-convex function $A(\vlambda)$.
By taking the gradient with respect to $\vlambda$ and substituting in \eqref{eq:A4}, we obtain that A6 is equivalent to
\[
(\vlambda -\vlambda_k)^T [ \bigtriangledown A(\vlambda) - \bigtriangledown A(\vlambda_k) ] \ge \alpha \|\vlambda - \vlambda_k\|^2,
\]
which is equivalent to strong-convexity of the function $A(\vlambda)$~\citep[][Theorem~2.1.9]{Nes04b}.


The method of~\citet{theis2015trust} corresponds to choosing $h = -\elbofinal$, $f = 0$, and $\mathbb{D}(\vlambda || \vlambda_k) := \dkls{}{q(\vz|\vlambda)}{q(\vz|\vlambda_k)}$ where $q$ is an exponential family distribution with natural parameters $\vlambda$. Since we assume $h$ to be convex, only limited cases of their approach are covered under our framework. The method of~\citet{Khan15nips} also uses the KL divergence and focuses on the deterministic case where $\widehat{\nabla} f(\vlambda) = \nabla f(\vlambda)$, but uses the split $-\elbofinal = f + h$ to allow for non-conjugate models.
In both of these models, A6 is satisfied when the Fisher matrix $\bigtriangledown^2 A(\vlambda)$ is positive-definite. This can be shown by using the definition of the KL divergence for exponential families \citep{nielsen2009statistical}:
\begin{equation}
\begin{aligned}
&\dkls{}{q(\vz|\vlambda)}{q(\vz|\vlambda_k)} \\
&\quad := A(\vlambda_k) - A(\vlambda) - [\bigtriangledown A(\vlambda)]^T (\vlambda_k - \vlambda). \label{eq:kl_exp}
\end{aligned}
\end{equation}
Taking the derivative with respect to $\vlambda$ and substituting in \eqref{eq:A4} with $\vlambda' = \vlambda_k$, we get the condition
\begin{align*}
(\vlambda - \vlambda_k)^T [\bigtriangledown^2 A(\vlambda)] (\vlambda - \vlambda_k) \ge \alpha \|\vlambda - \vlambda_k\|^2,
\end{align*}
which is satisfied when $\bigtriangledown^2 A(\vlambda)$ is positive-definite over a compact set for $\alpha$ equal to its lowest eigenvalue on the set.

Methods based on natural-gradient using iteration \eqref{eq:naturalgradient} (like SVI) correspond to using $h=0$, $f=-\elbofinal$, and the symmetric KL divergence. Assumption A1 to A5 are usually assumed for these methods and, as we show next, A6 is also satisfied. In particular, when $q$ is an exponential family distribution the symmetric KL divergence can be written as the sum of the Bregman divergence shown in \eqref{eq:breg} and the KL divergence shown in \eqref{eq:kl_exp}, 
\begin{align*}
&\dkls{sym}{q(\vz|\vlambda)}{q(\vz|\vlambda_k)} \nonumber\\
&:= \dkls{}{q(\vz|\vlambda_k)}{q(\vz|\vlambda)} + \dkls{}{q(\vz|\vlambda)}{q(\vz|\vlambda_k)} \nonumber\\
&= \mathbb{D}_A(\vlambda\| \vlambda_k) + \dkls{}{q(\vz|\vlambda)}{q(\vz|\vlambda_k)}
\end{align*}
where the first equality follows from the definition of the symmetric KL divergence and the second one follows from \eqref{eq:breg}. Since the two divergences in the sum satisfy A6, the symmetric KL divergence also satisfies the assumption.

\section{CONVERGENCE OF PG-SVI}

We first analyze the convergence rate of deterministic methods where the gradient is exact, $\widehat{\nabla} f(\vlambda) = \nabla f(\vlambda)$. This yields a simplified result that applies to a wide variety of existing variational methods. Subsequently, we consider the more general case where a stochastic approximation of the gradient is used.

\subsection{DETERMINISTIC METHODS}
We first establish the convergence under a fixed step-size when using the exact gradient. We use $C_0 = \elbofinal^* - \elbofinal(\vlambda_0)$ as the initial (constant) sub-optimality, and express our result in terms of the quantity 
\[
G_k := \frac{1}{\beta} (\vlambda_k -\vlambda_{k+1}),
\]
where $\vlambda_{k+1}$ is computed using~\eqref{eq:subproblem}.
%
\begin{prop} \label{corr:cnst_step}
Let A1, A2, A5, and A6 be satisfied. If we run $t$ iterations of~\eqref{eq:subproblem} with a fixed step-size $\beta_k = \beta = \alpha/L$ for all $k$ and an exact gradient $\nabla f(\vlambda)$, then we have
\begin{equation}
\label{eqtm21}
\begin{split}
& \min_{k \in \{0,1,\dots,t-1\}} \,  \norm{G_k}^2 \le \frac{2LC_0}{\alpha^2t} 
\end{split}
\end{equation}
\end{prop}
We give a proof in the appendix. 
Stating the result in terms of $G_k$ may appear to be unconventional, but this quantity is natural measure of first-order optimality. For example, consider the special case of gradient descent where $h = 0$ and $\mathbb{D}(\vlambda,\vlambda_k) = \frac{1}{2}\norm{\vlambda-\vlambda_k}^2$. In this case, $\alpha=1$ and $\beta_k = 1/L$, therefore we have  $\norm{G_k} = \norm{\nabla f(\vlambda_k)}$ and Proposition~\ref{corr:cnst_step} implies that $\min_{k}\norm{\nabla f(\vlambda_k)}^2$ has a convergence rate of $O(1/t)$. This means that the method converges at a sublinear rate to an approximate stationary point, which would be a global minimum in the special case where $f$ is convex. 

In more general settings, the quantity $G_k$ provides a generalized notion of first-order optimality for problems that may be non-smooth or use a non-Euclidean geometry. 
Further, if the objective is bounded below ($C_0$ is finite), this result implies that the algorithm converges to such a stationary point and also gives a rate of convergence of $O(1/t)$.


If we use a divergence with $\alpha>1$ then we can use a step-size larger than $1/L$ and the error will decrease faster than gradient-descent. 
To our knowledge, this is the first result that formally shows that natural-gradient methods can achieve faster convergence rates. The splitting of the objective into $f$ and $h$ functions is also likely to improve the step-size. Since $L$ only depends on $f$, sometimes it might be possible to reduce the Lipschitz constant by choosing an appropriate split.

We next give a more general result that allows a per-iteration step size.
\begin{prop} \label{thm:main}
If we choose the step-sizes $\beta_k$ to be such that $0<\beta_k\le 2\alpha/L$ with $\beta_k< 2\alpha/L$ for at least one $k$, then, 
\begin{equation}
\label{eqtm2}
\begin{split}
& \min_{k \in \{0,1\dots t-1\}} \, \norm{G_k}^2 \le \frac{C_0} {\sum_{k=0}^{t-1} \rnd{ \alpha \beta_k - L\beta_k^2/2}}
\end{split}
\end{equation}
\end{prop}
We give a proof in the appendix. For gradient-descent, the above result implies that we can use any step-size less than $2/L$, which agrees with the classical step-size choices for gradient and proximal-gradient methods. 

\subsection{STOCHASTIC METHODS}
We now give a bound for the more general case where we use a stochastic approximation of the gradient.
\begin{prop} \label{corr:cnst_step_stochastic}
Let A1-A6 be satisfied. If we run $t$ iterations of~\eqref{eq:subproblem} for a fixed step-size $\beta_k = \alpha_*/L$ (where $0<\gamma<2$ is a scalar) and fixed batch-size $M_k = M$ for all $k$ with a stochastic gradient $\widehat{\nabla} f(\vlambda)$, then we have 
\[
\myexpect_{R,\boldsymbol \xi} (\norm{G_R}^2) \leq  \sqr{\frac{2LC_0}{\alpha_*^2 t}  + \frac{ c\sigma^2}{M\alpha^*}}.
\]
where $c$ is a constant such that $c > 1/(2\alpha)$ and $\alpha_* := \alpha - 1/(2c)$. The expectation is taken with respect to the noise $\vxi := \{\vxi_0,\vxi_1,\ldots,\vxi_{t-1}\}$, and a random variable $R$ which follows the uniform distribution $Prob(R=k) = 1/t, \forall k\in \{0,1,2,\ldots,t-1\}$.
\end{prop}
Unlike the bound of Proposition~\ref{corr:cnst_step}, this bound depends on the noise variance $\sigma^2$ as well the mini-batch size $M$. In particular, as we would expect, the bound gets tighter as the variance gets smaller and as the size of our mini-batch grows. Notice that the dependence on the variance $\sigma^2$ is also improved if we have a favourable geometry that increases $\alpha^*$. Thus, we can achieve a higher accuracy by either increasing the mini-batch size or improving the geometry.
In the appendix we give a more general result that allows non-constant sequences of step sizes, 
although we found that constant step-sizes work better empirically.
Note that while stating the result in terms of a randomized iteration might seem strange, in practice we typically just take the last iteration as the approximate minimizer.

\section{CLOSED-FORM UPDATES FOR NON-CONJUGATE MODELS} \label{sec:example_of_pgsvi}

We now give an example where iteration \eqref{eq:subproblem} attains a closed-form solution. We expect such closed-form solution to exist for a large class of problems, including models where $q$ is an exponential-family distribution, but here we focus on the GP model discussed in Section \ref{sec:splitting}.



For the GP model, we rewrite the lower bound \eqref{eq:gp_lb} as
\begin{align}
-\elbofinal(\vm,\vV) := \underbrace{\sum_{n=1}^N f_n(m_n,v_n)}_{f(\boldsymbol{m},\boldsymbol{V})} + \underbrace{\dkls{}{q}{p}}_{h(\boldsymbol{m},\boldsymbol{V})} \label{eq:gp_lb_1}
\end{align}
where we've used  $q:=\gauss(\vz|\vm,\vV)$, $p:=\gauss(\vz|0,\vK)$, and $f_n(m_n,v_{n}):= -\mathbb{E}_{q}[\log p(y_n|z_n)]$ with $m_n$ being the  entry $n$ of $\vm$ and $v_n$ being the diagonal entry $n$ of $\vV$. We can compute a stochastic approximation of $f$ using \eqref{eq:gradient_approx} by randomly selecting an example $n_k$ (choosing $M=1$) and using a Monte Carlo gradient approximation of $f_{n_k}$. Using this approximation, the linearized term in \eqref{eq:subproblem} can be simplified to the following:
\begin{align}
\vlambda^T \sqr{ \widehat{\bigtriangledown} f(\vlambda_k)}  \nonumber 
&=  m_n \underbrace{N [ \nabla_{m_n} f_{n_k}(m_{n_k,k}, v_{n_k,k}) ]}_{:= \alpha_{n_k,k}} \nonumber \\
&\quad + v_{n} \underbrace{N [\nabla_{v_{n}} f_{n_k}(m_{n_k,k}, v_{n_k,k}) ]}_{ := 2\,\gamma_{n_k,k}} \nonumber\\
&= m_n \alpha_{n_k,k} + \half v_{n} \gamma_{n_k,k}
\end{align}
where $m_{n_k,k}$ and $v_{n_k,k}$ denote the value of $m_n$ and $v_n$ in the $k$'th iteration for $n=n_k$. 
By using the KL divergence as our divergence function in iteration \eqref{eq:subproblem}, and by denoting $\gauss(\vz|\vm_k,\vV_k)$ by $q_k$, we can express the two last two terms in \eqref{eq:subproblem} as a single KL divergence function as shown below:
\begin{equation*}
\begin{split}
&\vlambda^T \sqr{ \widehat{\bigtriangledown} f(\vlambda_k)} + h(\vlambda) + \frac{1}{\beta_k} \mathbb{D}(\vlambda\|\vlambda_k),\\
&= (m_n \alpha_{n,k} + \half v_{n} \gamma_{n,k}) + \dkls{}{q}{p} + \frac{1}{\beta_k} \dkls{}{q}{q_k}, \\
&= (m_n \alpha_{n,k} + \half v_{n} \gamma_{n,k}) + \frac{1}{1-r_k} \dkls{}{q}{p^{1-r_k} q_k^{r_k}},
\end{split}
\end{equation*}
where $r_k := 1/(1+\beta_k)$.
Comparing this to \eqref{eq:gp_lb_1}, we see that this objective is similar to that of a GP model with a Gaussian prior\footnote{Since $p$ and $q$ are Gaussian, the product is a Gaussian.} $p^{1-r_k} q_k^{r_k}$ and a linear Gaussian-like log-likelihood. Therefore, we can obtain closed-form updates for its minimization. 

The updates are shown below and a detailed derivation is given in the appendix.
\begin{align}
&\tvgamma_k = r_k \tvgamma_{k-1} + (1-r_k) \gamma_{n_k,k} \vone_{n_k} , \nonumber \\
&\vm_{k+1} = \vm_k - (1-r_k) (\vI - \vK\vA_k^{-1}) (\vm_k + \alpha_{n_k,k}\boldsymbol{\kappa}_{n_k}) , \nonumber\\
&v_{n_{k+1}, k+1} = \kappa_{n_{k+1},n_{k+1}} - \boldsymbol{\kappa}_{n_{k+1}}^T \vA_k^{-1} \boldsymbol{\kappa}_{n_{k+1}},  \label{eq:effupdate1}
\end{align}
where $\tvgamma_0$ is initialized to a small positive constant to avoid numerical issues, $\vone_{n_k}$ is a vector with all zero entries except $n_k$'th entry which is equal to 1, $\boldsymbol{\kappa}_k$ is $n_k$'th column of $\vK$, and $\vA_k := \vK + [\diag(\tvgamma_k)]^{-1}$.  
For iteration $k+1$, we use $m_{n_{k+1},k+1}$ and $v_{n_{k+1}, k+1}$ to compute the gradients $\alpha_{n_{k+1},k+1}$ and $\gamma_{n_{k+1},k+1}$, and run the above updates again. We continue until a convergence criteria is reached.

There are numerous advantages of these updates. First, We do not need to store the full covariance matrix $\vV$.
The updates avoid forming the matrix and only update $\vm$. This works because we only need one diagonal element in each iteration to compute the stochastic gradient $\gamma_{n_k,k}$.
For large $N$ this is a clear advantage since the memory cost is $O(N)$ rather than $O(N^2)$.
Second, computation of the mean vector $\vm$ and a diagonal entry of $\vV$ only require solving two linear equations, as shown in the second and third line of \eqref{eq:effupdate1}. In general, for a mini-batch of size $M$, we need a total of $2 M$ linear equations, which is a lot cheaper than an explicit inversion.
Finally, the linear equations at iteration $k+1$ are very similar to those at iteration $k$, since $\vA_k$ differ only at one entry from $\vA_{k+1}$. Therefore, we can reuse computations from the previous iteration to improve the computational efficiency of the updates.

\section{EXPERIMENTAL RESULTS}
In this section, we compare our method to many existing approaches such as SGD and four adaptive gradient-methods (ADAGRAD, ADADELTA, RMSprop, ADAM), as well as two variational inference methods for non-conjugate models (the delta method and Laplace method). We show results on Gaussian process classification \citep{Kuss05} and correlated topic models \citep{blei2007correlated}. The code to reproduce these experiments can be found on GitHub.\footnote{{\url{https://github.com/emtiyaz/prox-grad-svi}}}

\subsection{GAUSSIAN PROCESS CLASSIFICATION}
\begin{figure}[!t]
\center
\includegraphics[width=3.1in]{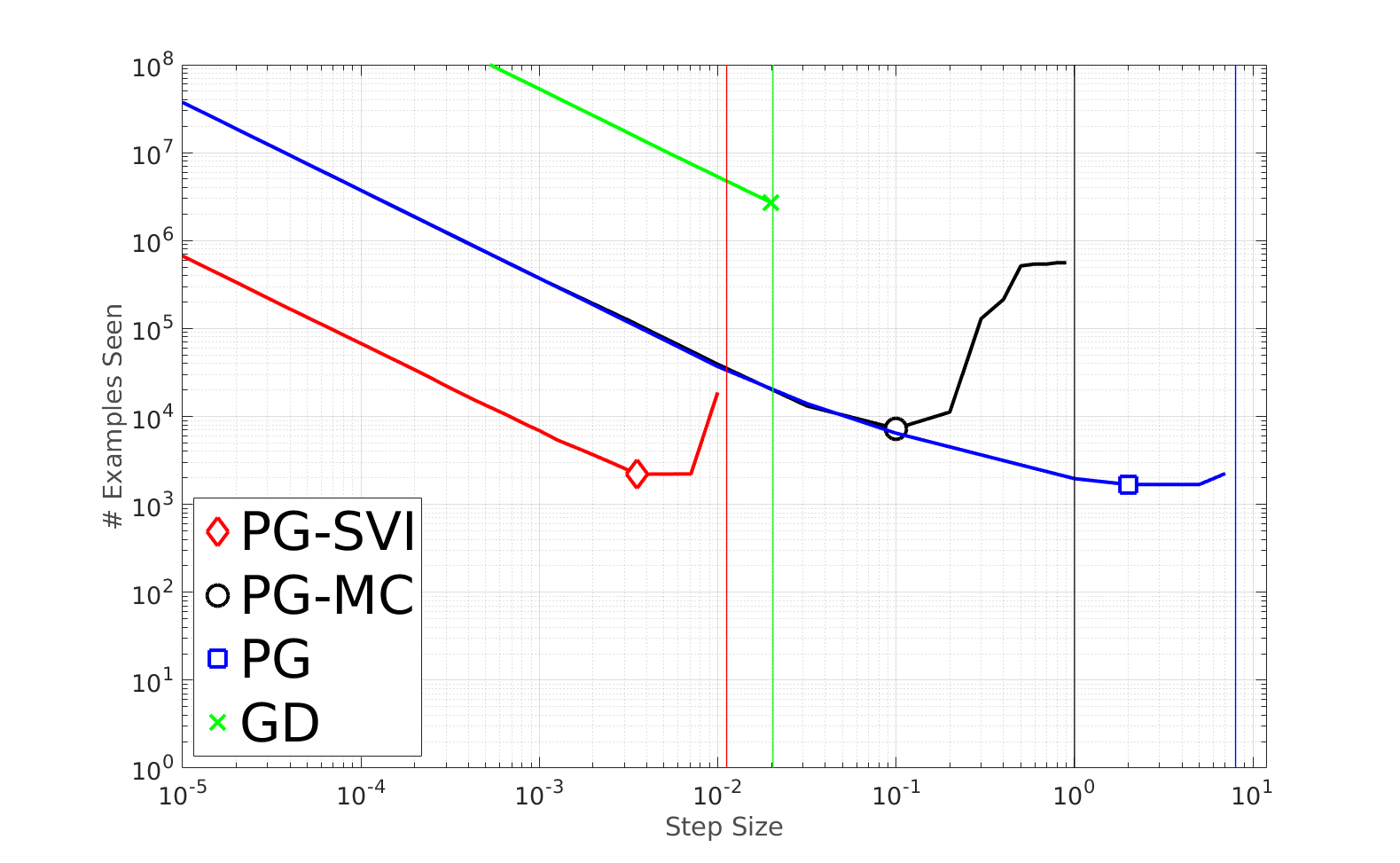} 
\caption{The number of examples required for convergence versus the step-size for binary GP classification for differnet methods. 
 The vertical lines show the step-size above which a method diverges.
}
\label{fig:step_size_gp}
\end{figure}

\begin{figure*}[!t]
\center
\subfigure{\includegraphics[width=2.2in]{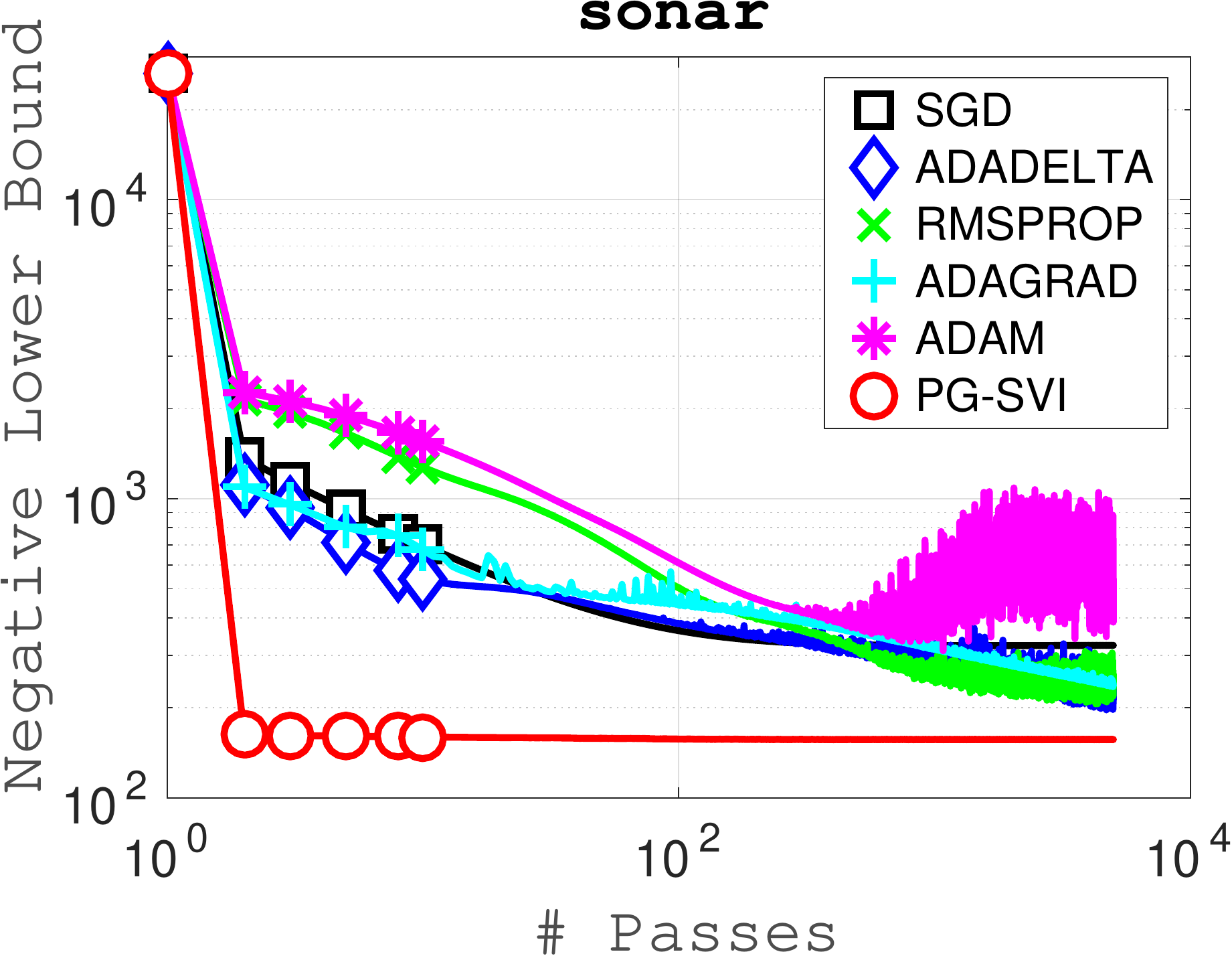}}
\hfill
\subfigure{\includegraphics[width=2.2in]{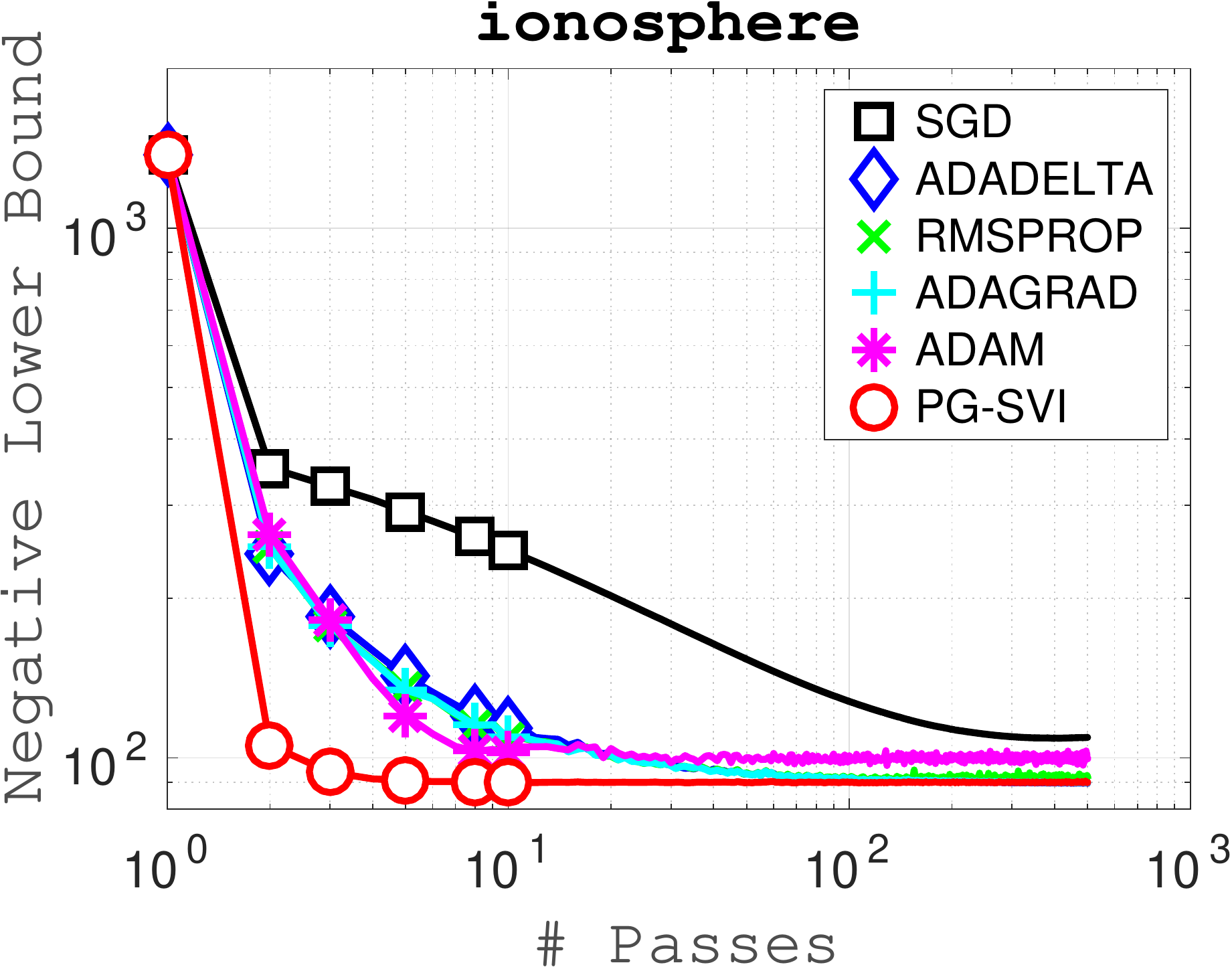}}
\hfill
\subfigure{\includegraphics[width=2.2in]{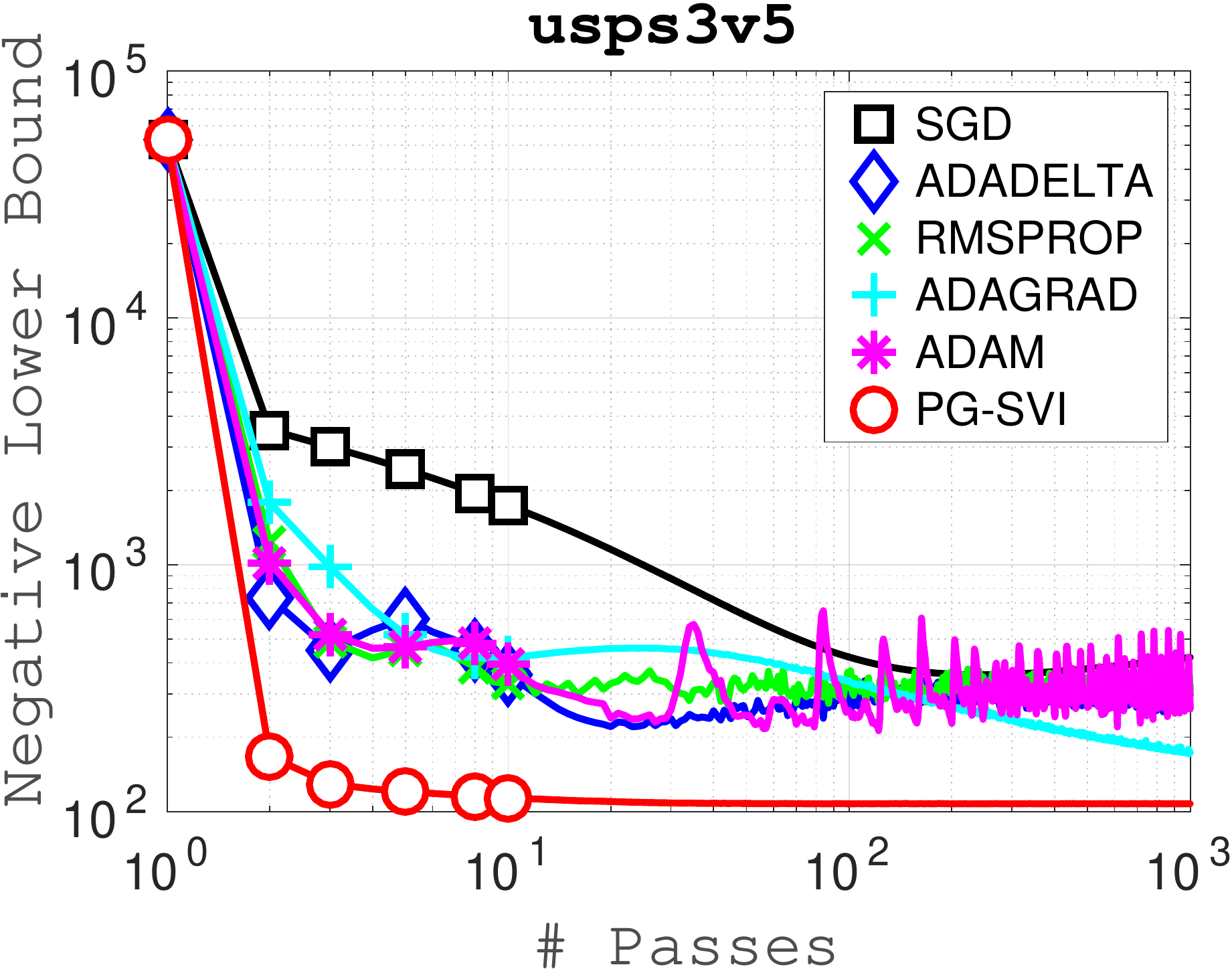}}
\hfill
\subfigure{\includegraphics[width=2.2in]{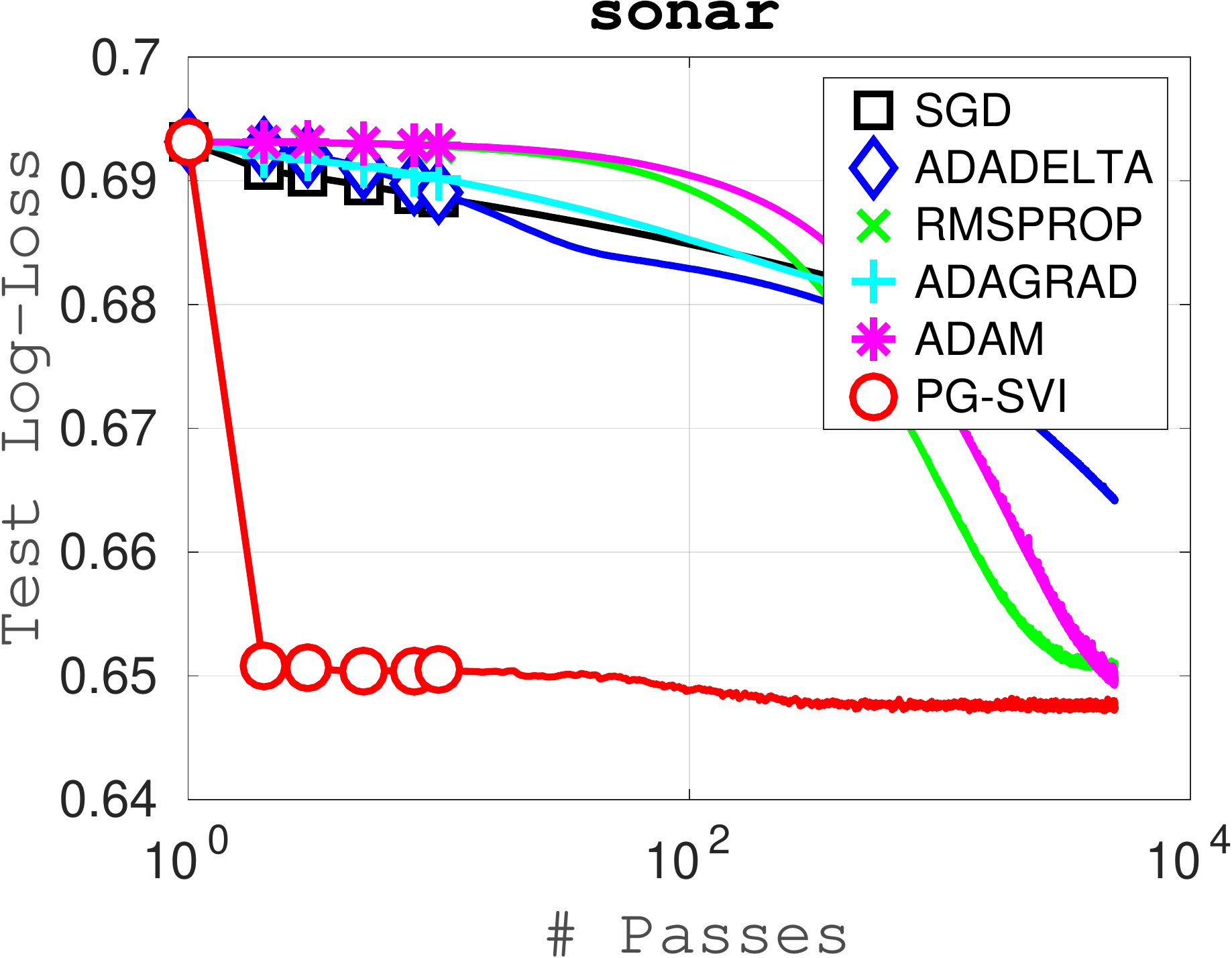}}
\hfill
\subfigure{\includegraphics[width=2.2in]{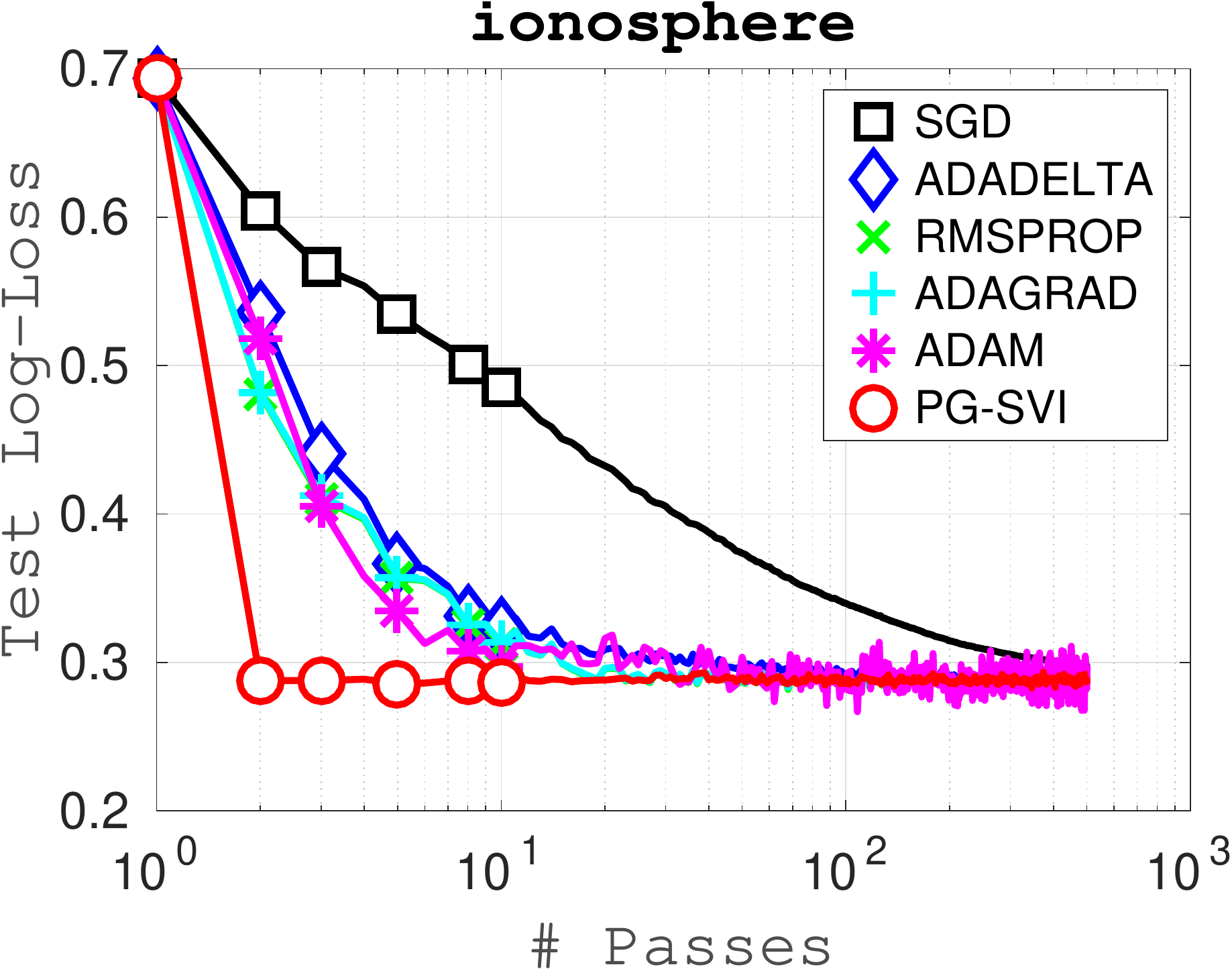}}
\hfill
\subfigure{\includegraphics[width=2.2in]{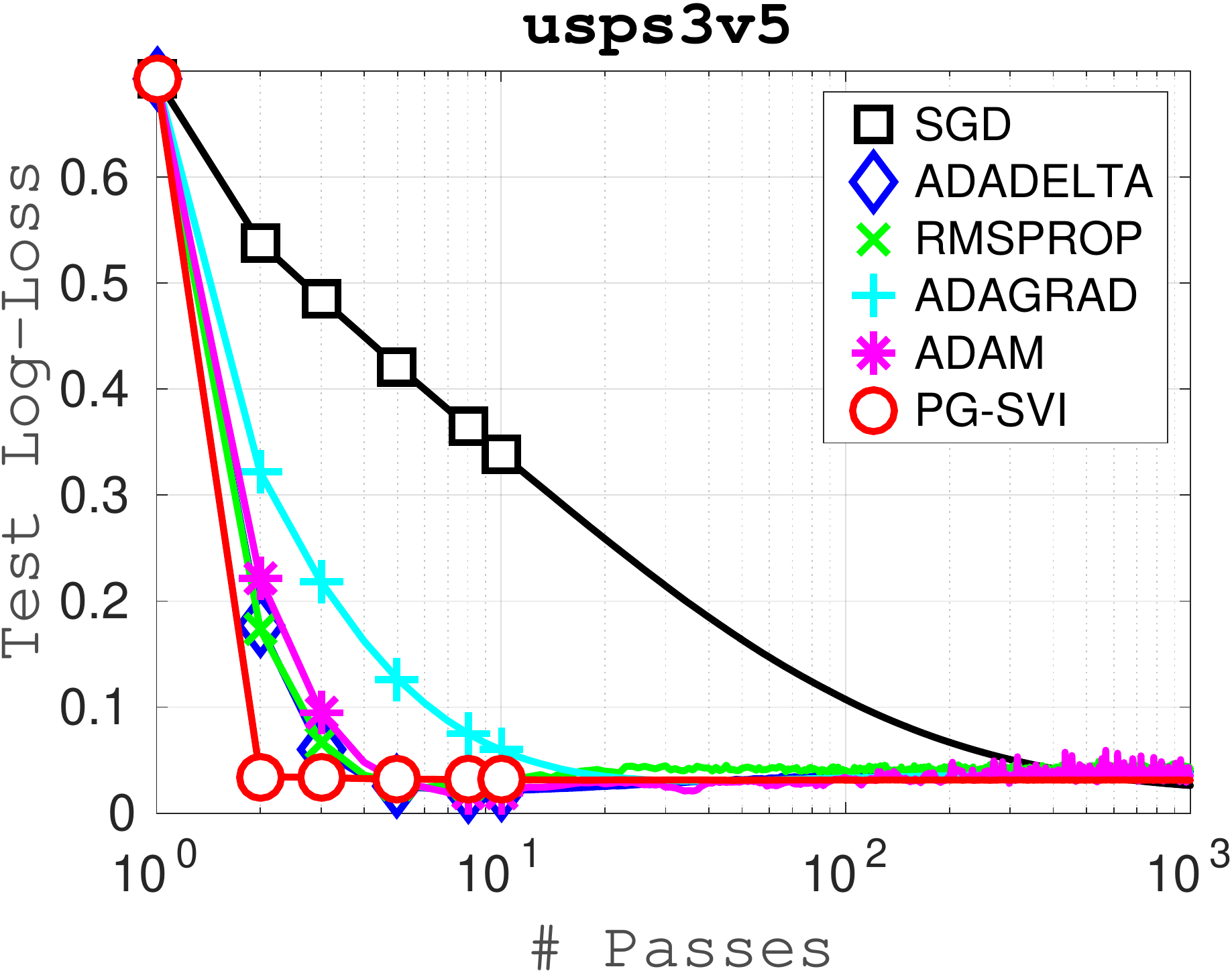}}
\caption{Comparison of different stochastic gradient methods for binary classification using GPs. 
Each column shows results for a dataset. The top row shows the negative of the lower bound while the bottom row shows the test log-loss. In each plot, the x-axis shows the number of passes made through the data. Markers are shown at 0, 1, 2, 4, 7, and 9 passes through the data. 
}
\label{fig:gp_class}
\end{figure*}

We first consider binary classification by using a GP model with a Bernoulli-logit likelihood on three datasets: Sonar, Ionosphere, and USPS-3vs5. These datasets can be found at the UCI data repository\footnote{{\url{https://archive.ics.uci.edu/ml/datasets.html}}} and their details are discussed in \cite{Kuss05}. For the GP prior, we use the zero mean-function and a squared-exponential covariance function with hyperparameters $\sigma$ and $l$ as defined in \cite{Kuss05} (see Eq. 33). We set the values of the hyperparameters using cross-validation. For the three datasets, the hyperparameters $(\log l, \log\sigma)$ are set to $(-1,6)$, $(1,2.5)$, and $(2.5,5)$, respectively.

\subsubsection{Performance Under a Fixed Step-Size}
In our first experiment, we compare the performance under a fixed step-size.
 The results also demonstrate the faster convergence of our method compared to gradient-descent methods. We compare the following four algorithms on the Ionosphere dataset: (1) batch gradient-descent (referred to as `GD'), (2) batch proximal-gradient algorithm (referred to as `PG'), (3) batch proximal-gradient algorithm with gradients approximated by using Monte Carlo (referred to as `PG-MC' and using $S=500$ samples), and (4) the proposed proximal-gradient stochastic variational-inference (referred to as `PG-SVI') method where stochastic gradients are obtained using \eqref{eq:gradient_approx} with $M=5$.

Figure \ref{fig:step_size_gp} shows the number of examples required for convergence versus the step-size. A lower number implies faster convergence. Convergence is assessed by monitoring the lower bound, and when the change in consecutive iterations do not exceed a certain threshold, we stop the algorithm. 

We clearly see that GD requires many more passes through the data, and proximal-gradient methods converge faster than GD. In addition, the upper bound on the step-size for PG is much larger than GD. This implies that PG can potentially take larger steps than the GD method. PG-SVI is surprisingly as fast as PG which shows the advantage of our approach over the approach of \cite{Khan15nips}. 

\subsubsection{Comparison with Adaptive Gradient Methods}
We also compare PG-SVI to SGD and four adaptive methods, namely ADADELTA \citep{zeiler2012adadelta}, RMSprop \citep{hintonTieleman}, ADAGRAD \citep{duchi2011adaptive}, and ADAM \citep{kingma2014adam}. The implementation details of these algorithms are given in the appendix. We compare the value of the lower bound versus number of passes through the data. We also compare the average log-loss on the test data,$- \sum_n \log \hat{p}_n/N_*$, where $\hat{p}_n = p(y_n|\sigma,l,\data_t)$ is the predictive probabilities of the test point $y_n$ given training data $\data_t$ and $N_*$ is the total number of test-pairs. A lower value is better for the log-loss, and a value of 1 is equal to the performance of random coin-flipping.

Figure \ref{fig:gp_class} summarizes the results. In these plots, lower is better for both objectives and one ``pass" means the number of randomly selected examples is equal to the total number of examples. 
Our method is much faster to converge than other methods. It always converged within 10 passes through the data while other methods required more than 100 passes.


\subsection{CORRELATED TOPIC MODEL}
We next show results for correlated topic model on two collections of documents, namely the NIPS and Associated Press (AP) datasets. The NIPS\footnote{{\url{https://archive.ics.uci.edu/}}} dataset contains 1500 documents from the NIPS conferences held between 1987 and 1999 (a vocabulary-size of 12,419 words and a total of around 1.9M words). The AP\footnote{\url{http://www.cs.columbia.edu/~blei/lda-c/index.html}} collection contains 2,246 documents from the Associated Press (a vocabulary-size of 10,473 words and a total of 436K observed words). We use 50\% of the documents for training and 50\% for testing.

We compare to the delta method and the Laplace method discussed in \citet{WangBlei}, and also to the original mean-field (MF) method of \citet{blei2007correlated}. For these methods, we use an available implementation.\footnote{\url{https://www.cs.princeton.edu/~chongw/resource.html}} All of these methods approximate the lower bound by using approximations to the expectation of log-sum-exp functions (see the appendix for details). We compare these methods to the two versions of our algorithm which do not use such approximations, but instead use a stochastic gradient as explain in Section \ref{sec:stoch_approx}. Specifically, we use the following two versions: one with full covariance (referred to as PG-SVI), and the other with a diagonal covariance (referred to as PG-SVI-MF). For both of these algorithms, we use a fixed step-size of 0.001, and a mini-batch size of 2 documents.

Following \citet{WangBlei} we compare the held-out log-likelihood which is computed as follows: a new test document $\vy$ is split into two halves $(\vy^1,\vy^2)$, then we compute the approximate posterior $q(\vz)$ to the posterior $p(\vz|\vy^1)$ and use this to compute the held-out log-likelihood for each $y_n \in \vy^2$ using
\begin{align}
\log p(y_n) \approx \log \int_z \sqr{\sum_{k=1}^K \beta_{n,k} \frac{e^{z_k}}{\sum_j e^{z_j}}}^{y_{n}} q(\vz) d\vz 
\end{align}
We use a Monte Carlo to this quantity by using a large number of samples from $q$ (unlike \cite{WangBlei} who approximate it by using the Delta method). We report the average of this quantity over all words in $\vy^2$.

\begin{figure}[!t] 
\center
\subfigure{\includegraphics[height=1.3in]{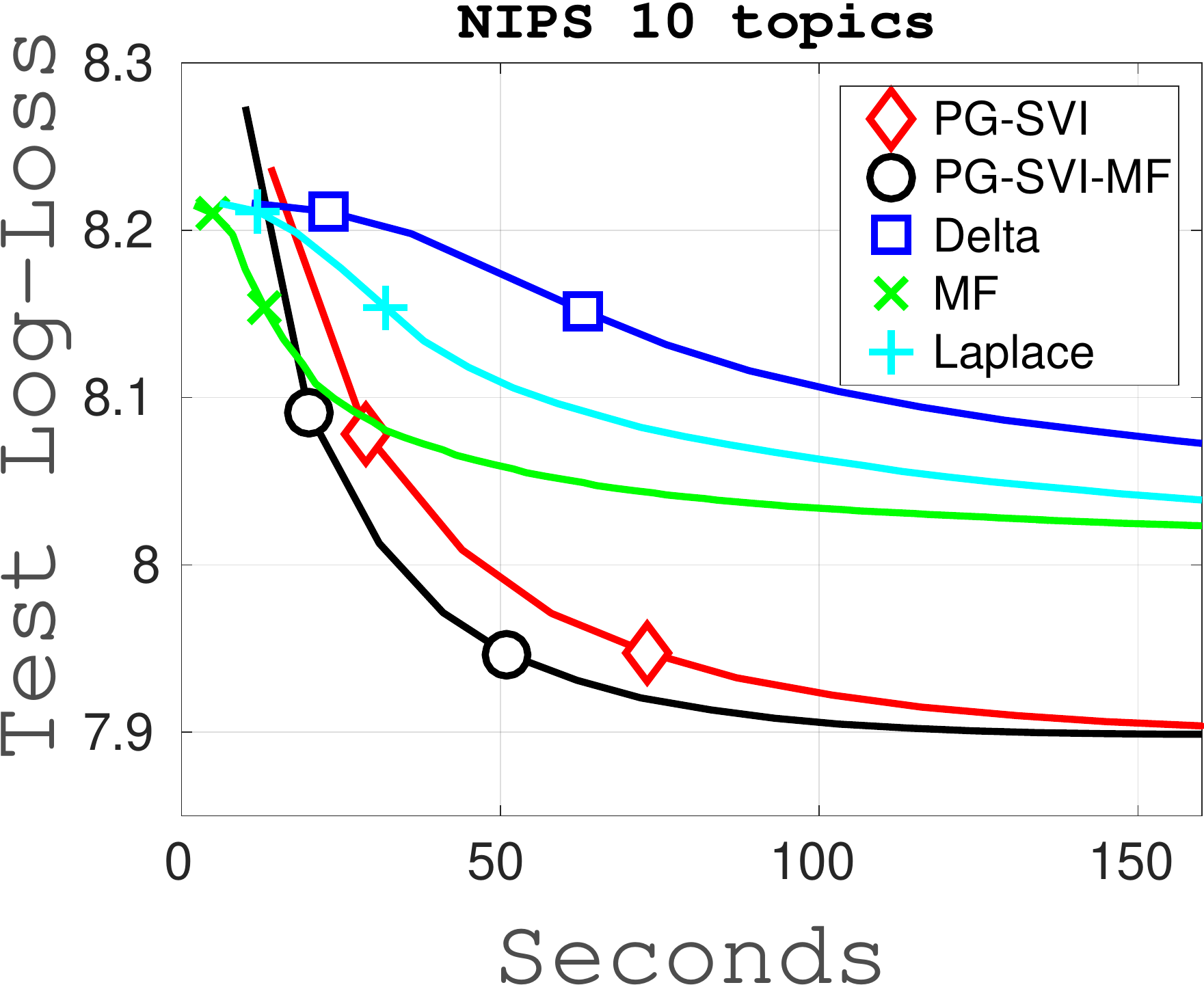} }
\hfill
\subfigure{\includegraphics[trim={1.1cm 0 0 0},clip, height=1.3in]{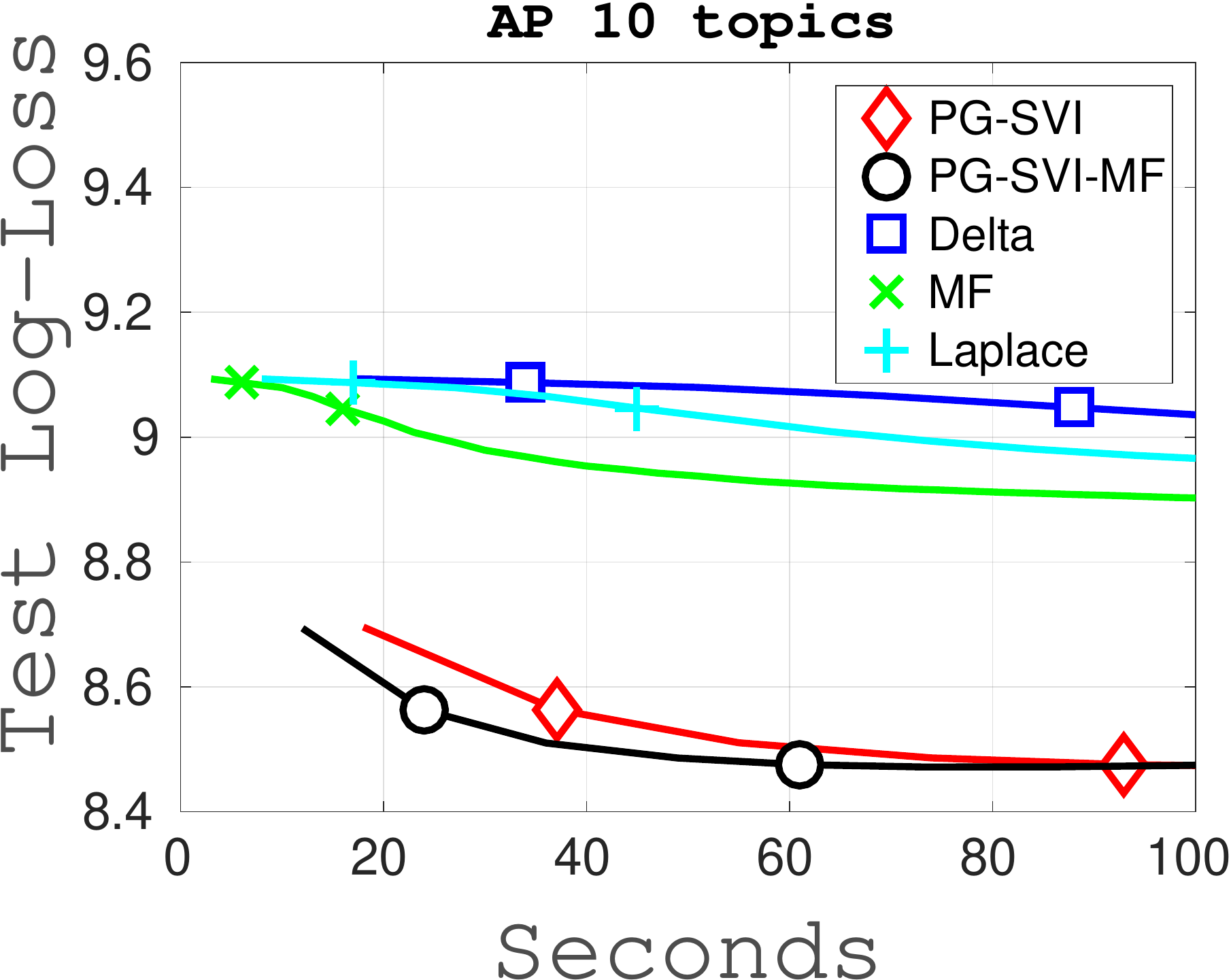} \label{fig:AP}}
\caption{Results on NIPS (left) and AP (right) datasets using correlated topic model with 10 topics. We plot the negative of the average held-out log-likelihood versus time. Markers are shown at iterations after second and fifth passes through the data. 
}
\label{fig:nips}
\end{figure}

Figure~\ref{fig:nips} shows the negative of the average held-out log-likelihood versus time for 10 topics (lower values are better). 
We see that methods based on proximal-gradient algorithm converge a little bit faster than the existing methods. More importantly, they achieves better performance. This could be due to the fact that we do not approximate the expectation of the log-sum-exp function, unlike the delta and Laplace method. We obtained similar results when we used a different number of topics. 


\section{DISCUSSION}

This work has made two contributions. First, we proposed a new variational inference method that combines variable splitting, stochastic gradients, and general divergence functions. This method is well-suited for a huge variety of the variational inference problems that arise in practice, and we anticipate that it may improve over state of the art methods in a variety of settings.
Our second contribution is a theoretical analysis of the convergence rate of this general method. Our analysis generalizes existing results for the mirror descent algorithm in optimization, and establishes convergences rates of a variety of existing variational inference methods. Due to its generality we expect that this analysis could be useful to establish convergence rates of other algorithms that we have not thought of, perhaps beyond the variational inference settings we consider in this work. However, an open problem that is also discussed by~\citet{ghadimi2014mini} it to esbatlish convergence to an arbitrary accuracy with a fixed batch size.

One issue that we have not satisfactorily resolved is giving a theoretically-justified way to set the step-size in practice; our analysis only indicates that it must be sufficiently small. However, this problem is common in many methods in the literature and our analysis at least suggests the factors that should be taken into account. 
Another open issue is the applicability our method to many other latent variable models; in this paper we have shown applications to variational-Gaussian inference, but we expect that our method should result in simple updates for a larger class of latent variable models such as non-conjugate exponential family distribution models. Additional work on these issues will improve usability of our method.

\newpage
\bibliography{paper}
\bibliographystyle{icml2016}

\newpage
\onecolumn
\begin{appendix}
\def\norms#1{\|#1\|^2}
\def\norm#1{\|#1\|}
\def\gt{\tilde{g}_{\lambda,k}}
\def\ex#1{\mathbb E [#1]}

\section{Examples of Splitting for Variational-Gaussian Inference}
We give detailed derivations for the splitting-examples shown in Section 3.1 in the main paper.
As in the main paper, we denote the Gaussian posterior distribution by $q(\vz|\vlambda) := \gauss(\vz|\vm,\vV)$, so that $\vlambda = \{\vm,\vV\}$ with $\vm$ being the mean and $\vV$ being the covariance matrix.

\subsection{Gaussian Process (GP) Models}
Consider GP models for $N$ input-output pairs $\{y_n,\vx_n\}$ indexed by $n$.
Let $z_n := f(\vx_n)$ be the latent function drawn from a GP with a zero-mean function and a covariance function $\kappa(\vx,\vx')$. We denote the Kernel matrix obtained on the data $\vx_n$ for all $n$ by $\vK$.

We use a non-Gaussian likelihood $p(y_n|z_n)$ to model the output, and assume that each $y_n$ is independently sampled from this likelihood given $\vz$. The joint-distribution over $\vy$ and $\vz$ is shown below:
\begin{align}
p(\vy,\vz) = \prod_{n=1}^N p(y_n|z_n) \gauss(\vz|0,\vK)
\end{align}
The ratio required for the lower bound is shown below, along with the split, where non-Gaussian terms are in $\tp_d$ and Gaussian terms are in $\tp_e$:
\begin{align}
\frac{p(\vy,\vz)}{q(\vz|\vm,\vV)} = \underbrace{\prod_{n=1}^N p(y_n|z_n)}_{\tp_d(\mathbf{z}|\boldsymbol{\lambda})} \underbrace{\frac{\gauss(\vz|0,\vK)}{\gauss(\vz|\vm,\vV)}}_{\tp_e(\mathbf{z}|\boldsymbol{\lambda})}.
\end{align}
By substituting in Eq. 1 of the main paper, we can obtain the lower bound $\elbofinal$ after a few simplifications, as shown below:
\begin{align}
\elbofinal(\vm,\vV) &:= \myexpect_{q(\mathbf{z})} \sqr{ \log \frac{p(\vy,\vz)}{q(\vz|\vm,\vV)} } , \\
&= \myexpect_{q(\mathbf{z})} \sqr{ \sum_{n=1}^N \log p(y_n|z_n)} + \myexpect_{q(\mathbf{z})} \sqr{\log \frac{\gauss(\vz|0,\vK)}{\gauss(\vz|\vm,\vV)} }  , \\
&= \underbrace{ \sum_{n=1}^N \mathbb{E}_{q}[\log p(y_n|z_n)] }_{-f(\boldsymbol{\lambda})} - \underbrace{ \dkls{}{\gauss(\vz|\vm,\vV)}{\gauss(\vz|0,\vK)}}_{h(\boldsymbol{\lambda})}. 
\end{align}
The assumption A2 is satisfied since the KL divergence is convex in both $\vm$ and $\vV$. This is clear from the expression of the KL divergence:
\begin{align}
D_{KL}\sqr{\gauss(\vz|\vm,\vV) || \gauss(\vz|0,\vK) } &= \half[-\log|\vV\vK^{-1}| + \trace(\vV\vK^{-1}) + \vm^T \vK^{-1} \vm - D]
\end{align}
where $D$ is the dimensionality of $\vz$.
Convexity w.r.t. $\vm$ follows from the fact that the above is quadratic in $\vm$ and $\vK$ is positive semi-definite. Convexity w.r.t. $\vV$ follows due to concavity of $\log|\vV|$ (trace is linear, so does not matter).

Assumption A1 depends on the choice of the likelihood $p(y_n|z_n)$, but is usually satisfied. The simplest example is a Gaussian likelihood for which the function $f$ takes the following form:
\begin{align}
f(\vm,\vV) &= \sum_{n=1}^N \myexpect_q[-\log p(y_n|z_n)] = \sum_{n=1}^N \myexpect_q[-\log \gauss(y_n|z_n,\sigma^2)] \\
&= \sum_{n=1}^N \half \log (2\pi\sigma^2) + \frac{1}{2\sigma^2} \sqr{ (y_n-m_n)^2 + v_n} 
\end{align}
where $m_n$ is the $n$'th element of $\vm$ and $v_n$ is the $n$'th diagonal entry of $\vV$. This clearly satisfies A1, since the objective is quadratic in $\vm$ and linear in $\vV$.

Here is an example where A1 is not satisfied: for Poisson likelihood $\log p(y_n|z_n) = \exp[y_nz_n-e^{z_n}] / y_n!$ with rate parameter equal to $e^{z_n}$, the function $f$ takes the following form:
\begin{align}
f(\vm,\vV) &= \sum_{n=1}^N \myexpect_q[-\log p(y_n|z_n)] = \sum_{n=1}^N [- y_nm_n + e^{m_n+v_n/2} + \log (y_n!) ]
\end{align}
whose derivative is not Lipschitz continuous since exponential functions are not Lipschitz. 

\subsection{Generalized Linear Models (GLMs)}
We now describe a split for generalized linear models. We model the output $y_n$ by using an exponential family distribution whose natural-parameter is equal to $\eta_n := \vx_n^T\vz$. Assuming a standard Gaussian prior over $\vz$, the joint distribution can be written as follows:
\begin{align}
p(\vy,\vz) := \prod_{n=1}^N p(y_n|\vx_n^T\vz) \gauss(\vz|0,\vI) 
\end{align}
A similar split can be obtained by putting non-conjugate terms $p(y_n|\vx_n^T\vz)$ in $\tp_d$ and the rest in $\tp_e$:
\begin{align*}
\frac{p(\vy,\vz)}{q(\vz|\vlambda)} = \underbrace{\prod_{n=1}^N p(y_n|\vx_n^T\vz)}_{\tp_d(\mathbf{z}|\boldsymbol{\lambda})} \underbrace{\frac{\gauss(\vz|0,\vI)}{\gauss(\vz|\vm,\vV)}}_{\tp_e(\mathbf{z}|\boldsymbol{\lambda})}.
\end{align*}
The lower bound can be shown to be the following:
\begin{align}
\elbofinal(\vm,\vV) := \underbrace{ \sum_{n=1}^N \mathbb{E}_{q}[\log p(y_n|\vx_n^T \vz)] }_{-f(\boldsymbol{\lambda})} - \underbrace{ \dkls{}{\gauss(\vz|\vm,\vV)}{\gauss(\vz|0,\vI)}}_{h(\boldsymbol{\lambda})}.
\end{align}
which is very similar to the GP case. Therefore, Assumptions A1 and A2 will follow with similar arguments.

\subsection{Correlated Topic Model (CTM)}
  We consider text documents with a vocabulary size $N$.
  Let $\vz$ be a length $K$ real-valued vector which follows a Gaussian distribution shown in \eqref{eq:ctm1}.
  Given $\vz$, a topic $t_{n}$ is sampled for the $n$'th word using a multinomial distribution shown in \eqref{eq:ctm2}.
  Probability of observing a word in the vocabulary is then given by \eqref{eq:ctm3}.
  \begin{align}
  p(\vz|\vtheta) &= \gauss(\vz|\vmu,\vSigma),  \label{eq:ctm1}\\
  p(t_{n}=k|\vz) &= \frac{\exp(z_{k})}{\sum_{j=1}^K\exp(z_{j})},  \label{eq:ctm2}\\
  p(\textrm{Observing a word v}|t_{n},\vtheta) &= \beta_{v,t_{n}} . \label{eq:ctm3}
  \end{align}
  Here $\vbeta$ is a $N\times K$ real-valued matrix with non-negative entries and columns that sum to 1.
  The parameter set for this model is given by $\vtheta = \{\vmu, \vSigma, \vbeta\}$.
We can marginalize out $t_n$ and obtain the data-likelihood given $\vz$,
\begin{align}
  p(\textrm{Observing a word v}|\vz,\vtheta) &= \sum_{k=1}^K p(\textrm{Observing a word v}|t_{n}=k,\vtheta) p(t_{n}=k|\vz) , \\
  &= \sum_{k=1}^K  \beta_{vk} \frac{e^{z_{k}}}{\sum_{j=1}^K e^{z_{j}}} . \label{eq:CTMlik}
  \end{align}
Given that we observe $n$'th word $y_n$ times, we can write the following joint distribution:
\begin{align}
p(\vy,\vz) := \prod_{n=1}^N \sqr{\sum_{k=1}^K \beta_{n,k} \frac{e^{z_k}}{\sum_j e^{z_j}}}^{y_{n}} \gauss(\vz|\vmu,\vSigma)
\end{align}
We can then use the following split:
\begin{align*}
\frac{p(\vy,\vz)}{q(\vz|\vlambda)} = \underbrace{\prod_{n=1}^N \sqr{\sum_{k=1}^K \beta_{n,k} \frac{e^{z_k}}{\sum_j e^{z_j}}}^{y_{n}} }_{\tp_d(\mathbf{z}|\boldsymbol{\lambda})} \underbrace{\frac{\gauss(\vz|\vmu,\vSigma)}{\gauss(\vz|\vm,\vV)}}_{\tp_e(\mathbf{z}|\boldsymbol{\lambda})},
\end{align*}
where $\vmu,\vSigma$ are parameters of the Gaussian prior and $\beta_{n,k}$ are parameters of $K$ multinomials.

The lower bound is shown below:
\begin{align}
\elbofinal(\vm,\vV) &:=  \sum_{n=1}^N y_n \crl{ \mathbb{E}_{q}\sqr{ \log \rnd{\sum_{k=1}^K \beta_{n,k} e^{z_k}} }} - W \myexpect_q \crl{\log \sqr{ \sum_{j=1}^K e^{z_j} }}  \nonumber\\
&\quad\quad -  \dkls{}{\gauss(\vz|\vm,\vV)}{\gauss(\vz|0,\vI)}.
\end{align}
where $W = \sum_n y_n$ is the total number of words. The top line is the function $[-f(\vlambda)]$ while the bottom line is $[-h(\vlambda)]$.

There are two intractable expectations in $f$, each involving expectation of a log-sum-exp function. Wang and Blei (2013) use the Delta method and Laplace method to approximate these expectations. In contrast, in PG-SVI algorithm, we use Monte Carlo to approximate the gradient of these functions.

\section{Proof of Proposition 1 and 2}
We first prove the Proposition 2. Proposition 1 is obtained as a special case of it. Our proof technique is borrowed from Ghadimi et. al. (2014). We extend their results to general divergence functions.

We denote the proximal projection at $\vlambda_k$ with gradient $\vg$ and step-size $\beta$ by,
\begin{align}
 &\mathcal{P}(\vlambda_k,\vg,\beta) := \frac{1}{\beta} (\vlambda_k -\vlambda_{k+1}), \label{eq:def_prox}\\
 &\quad\quad \textrm{ where }
\vlambda_{k+1} = \arg\min_{\boldsymbol{\lambda}\in\mathcal{S}} \,\,  \vlambda^T \vg + h(\vlambda) + \frac{1}{\beta} \ddd{}{\vlambda}{\vlambda_k}  \label{eq:prox_proj_l1}. 
\end{align}
The following lemma gives a bound on the norm of $\mathcal{P}(\vlambda_k,\vg,\beta)$.
\begin{lemma} \label{lemma:1}
The following holds for any $\vlambda_k\in\mathcal{S}$, any real-valued vector $\vg$ and $\beta>0$.
\begin{align}
&\vg^T \mathcal{P}(\vlambda_k, \vg,\beta)  \ge \alpha ||\mathcal{P}(\vlambda_k, \vg,\beta)||^2 + \frac{1}{\beta} [h(\vlambda_{k+1}) - h(\vlambda_k)]
\end{align}
\end{lemma}

\begin{proof}
By taking the gradient of $\vlambda^T \vg + \frac{1}{\beta} \ddd{}{\vlambda}{\vlambda_k}$ and picking any sub-gradient $\nabla h$ of $h$ at $\vlambda_{k+1}$, the corresponding sub-gradient of the right hand side of (\ref{eq:prox_proj_l1}) is given as follows:
\begin{align}
\vg + \nabla h(\vlambda_{k+1}) + \frac{1}{\beta} \nabla_{\lambda} \ddd{}{\vlambda_{k+1}}{\vlambda_k} .
\end{align}
We use this to derive the optimality condition of (\ref{eq:prox_proj_l1}). For any $\vlambda$, the following holds from the optimality condition:
\begin{align}
(\vlambda - \vlambda_{k+1})^T \Bigg[&\vg + \nabla h(\vlambda_{k+1}) + \frac{1}{\beta}  \nabla_{\lambda} \ddd{}{\vlambda_{k+1}}{\vlambda_k} \Bigg]  \ge 0 .
\end{align}
Letting $\vlambda = \vlambda_k$,
\begin{align}
(\vlambda_k - \vlambda_{k+1})^T \Bigg[&\vg + \nabla h(\vlambda_{k+1}) + \frac{1}{\beta} \nabla_{\lambda} \ddd{}{\vlambda_{k+1}}{\vlambda_k} \Bigg]  \ge 0 ,
\end{align}
which implies,
\begin{align}
\vg^T(\vlambda_k -\vlambda_{k+1}) &\ge \frac{1}{\beta} (\vlambda_{k+1} - \vlambda_k)^T  \nabla_{\lambda} \ddd{}{\vlambda_{k+1}}{\vlambda_k} + h(\vlambda_{k+1}) - h(\vlambda_k) , \\
&\ge \frac{\alpha}{\beta} ||\vlambda_{k+1} - \vlambda_k||^2 + h(\vlambda_{k+1}) - h(\vlambda_k) .
\end{align}
The first line follows from Assumption A2 (convexity of $h$), and the second line follows from Assumption A6.
\end{proof}

Now, we are ready to prove Proposition 2: 
\begin{proof}
Let $\gt:=\mathcal{P}(\vlambda_{k},\nabla f(\vlambda_k),\beta_{k})$. 
Since $\nabla f$ is L-smooth (Assumption A1),  for any $k=0,1,\ldots, t-1$ we have, 
\begin{align*}
f(\vlambda_{k+1}) &\leq f(\vlambda_{k}) + \left< \nabla f(\vlambda_k), \vlambda_{k+1}-\vlambda_k\right>+ \frac L 2 \norms{\vlambda_{k+1}-\vlambda_{k}}, \\  
& = f(\vlambda_{k}) - \beta_k\left< \nabla f(\vlambda_k), \gt \right>+ \frac L 2 \beta_k^2\norms{\gt}, \\
&\leq f(\vlambda_{k}) - \beta_k\alpha\norms{\gt}-[h(\vlambda_{k+1}) - h(\vlambda_k)]+ \frac L 2 \beta_k^2\norms{\gt}.
\end{align*}
The second line follows from the definition of $\mathcal{P}$ and the last line is due to Lemma 1. Rearranging the terms and using $-\elbofinal  = f + h$ we get: 
\begin{align*}
&-\elbofinal(\vlambda_{k+1}) + \elbofinal(\vlambda_{k}) \leq -[\beta_k\alpha - \frac L 2 \beta_k^2]\norms{\gt} ,\\
\Rightarrow\quad\quad & \elbofinal(\vlambda_{k+1}) -\elbofinal(\vlambda_{k}) \geq [\beta_k\alpha - \frac L 2 \beta_k^2]\norms{\gt} .
\end{align*}
Summing these term for all $k=0,1,\dots t-1$, we get the following: 
\begin{align*}
\elbofinal(\vlambda_{t-1}) -\elbofinal(\vlambda_{0}) \geq \sum_{k=0}^{t-1} [\beta_k\alpha - \frac L 2 \beta_k^2]\norms{\gt} .
\end{align*}
By noting that the global maximum of the lower bound always upper bounds any other value, we get $\elbofinal(\vlambda_{*}) -\elbofinal(\vlambda_{0}) \geq \elbofinal(\vlambda_{t-1}) -\elbofinal(\vlambda_{0})$. Using this, 
\begin{align*}
&\elbofinal(\vlambda_{*}) -\elbofinal(\vlambda_{0}) \geq \sum_{k=0}^{t-1}[\beta_k\alpha - \frac L 2 \beta_k^2]\norms{\gt} , \\ 
\Rightarrow\quad\quad &\min_{k=0,1,\dots, t-1 }\norms{\gt} [{\sum_{k=0}^{t-1}[\beta_k\alpha - \frac L 2 \beta_k^2]}] \leq  \elbofinal(\vlambda_{*}) -\elbofinal(\vlambda_{0}) .
\end{align*}
Since we assume at least one of $\beta_k < 2\alpha/L$, we can divide by the summation term, to get the following: 
\[
\min_{k=0,1,\dots,t-1 }\norms{\gt} \leq  \frac {\elbofinal(\vlambda_{*}) -\elbofinal(\vlambda_{1})}{{\sum_{k=0}^{t-1} [\beta_k\alpha - \frac L 2 \beta_k^2]} },\]
which proves  Proposition 2.
\end{proof}

Proposition 1 can be obtained by simply plugging in $\beta_k = \alpha/L$,
\[
\min_{k=0,1,\dots,t-1 }\norms{\gt} \leq  \frac {C_0}{{\sum_{k=0}^{t-1}[\frac{\alpha^2}{L} - \frac{\alpha^2}{2L} ]} } = \frac{2C_0L}{\alpha^2 t} .
\]

\section{Proof of Proposition 3}
We will first prove the following theorem, which gives a similar result to Proposition 2 but for a stochastic gradient $\widehat{\nabla} f$. 

\begin{thm} \label{thm:main2}
If we choose the step-size $\beta_k$ such that $0<\beta_k\le 2\alpha_*/L$ with $\beta_k< 2\alpha_*/L$ for at least one $k$, then, 
\begin{equation}
\label{eqtm22}
\begin{split}
&\myexpect_{R,\boldsymbol{\xi}} [\norm{G_R}^2]
\le \frac{C_0+ \frac{c\sigma^2}{2} \sum_{k=0}^{t-1} \frac{\beta_k}{M_k}} {\sum_{k=0}^{t-1} \rnd{ \alpha_* \beta_k - L\beta_k^2/2}}.
\end{split}
\end{equation}
where the expectation is taken over $R\in \{0,1,2,\ldots,t-1\}$ which is a discrete random variable drawn from the probability mass function
\begin{align}
Prob(R=k) = \frac{\alpha_*\beta_k - L\beta_k^2/2}{ \sum_{k=0}^{t-1} \rnd{\alpha_*\beta_k- L\beta_k^2/2} }, \nonumber
\end{align}
and over $\vxi := \{\vxi_1,\vxi_2,\ldots,\vxi_{t-1}\}$ with $\vxi_k$ is the noise in the stochastic approximation $\widehat{\nabla} f$. 
\end{thm}

\begin{proof}
Let $\gt:=\mathcal{P}(\vlambda_{k},\widehat{\nabla} f(\lambda_{k}),\beta_{k}), \, \delta_k:=\widehat{\nabla} f(\lambda_{k})- \nabla f(\vlambda_k)$. Since $\nabla f$ is L-smooth, for any $k=0,1,\dots, t$ we have, 
\begin{align}
f(\vlambda_{k+1}) &\leq f(\vlambda_{k}) + \left< \nabla f(\vlambda_k), \vlambda_{k+1}-\vlambda_k\right>+ \frac L 2 \norms{\lambda_{k+1}-\vlambda_{k}}\\  
& = f(\vlambda_{k}) - \beta_k\left< \nabla f(\vlambda_k), \gt \right>+ \frac L 2 \beta_k^2\norms{\gt}\\
& = f(\vlambda_{k}) - \beta_k\left<\widehat{\nabla} f(\lambda_{k}), \gt \right>+ \frac L 2 \beta_k^2\norms{\gt}+\beta_k\left<\delta_k, \gt \right>
\end{align}
where we have used the definition of $\gt$ and $\delta_k$. 
Now using Lemma \ref{lemma:1} on the second term and Cauchy-Schwarz for the last term, we get the following:
\begin{align}
f(\vlambda_{k+1}) &\leq f(\vlambda_{k}) -\left[ \alpha \beta_k\norms{\gt}+h(\vlambda_{k+1})-h(\vlambda_{k}) \right]  + \frac L 2 \beta_k^2\norms{\gt}+\beta_k\norm{\delta_k}\norm{\gt} 
\end{align}
After rearranging and using Young's inequality $\norm{\delta_k}\norm{\gt} \leq (c/2)\norm{\delta_k}^2+1/(2c) \norm{\gt}^2$ given a constant $c>0$, we get 
\begin{align}
-\elbofinal(\vlambda_{k+1}) &\leq -\elbofinal(\vlambda_{k}) - \alpha \beta_k\norms{\gt} + \frac L 2 \beta_k^2\norms{\gt}+ \frac {\beta_k}{ 2c} \norms{\gt}+\frac {\beta_k c} 2 \norms{\delta_k}\\
&= -\elbofinal(\vlambda_{k}) - \left( (\alpha-1/(2c))\beta_k-\frac L 2 \beta_k^2\right)\norms{\gt}+ \frac {c\beta_k} 2 \norms{\delta_k}
\end{align}

Now considering $c > 1/(2\alpha)$, $\alpha_* = \alpha - 1/(2c)$ and $\beta_k \leq \frac {2\alpha_*} {L}$, and summing up both sides for iteration $k=0,1\dots, t-1$, we obtain
\begin{align}
\sum_{k=0}^{t-1}  &\left( \alpha_*\beta_k-\frac L 2 \beta_k^2\right)\norms{\gt} \leq \elbofinal^* - \elbofinal(\vlambda_0)+ \sum_{k=0}^{t-1} \frac {c\beta_k} 2 \norms{\delta_k}
\end{align}
Now by taking expectation w.r.t. $\vxi$ on both sides and using the fact that $\myexpect_{\boldsymbol{\xi}}{\norms{\delta_k}} \leq \frac {\sigma^2}{M_k}$ by assumption $A3$ and $A4$, we get 
\begin{align}
\sum_{k=0}^{t-1}  &\left( \alpha_*\beta_k-\frac L 2 \beta_k^2\right)\myexpect_{\boldsymbol \xi}{\norms{\gt}} \leq C_0 + \frac {c\sigma^2}{2}\sum_{k=0}^{t-1} \frac {\beta_k} {M_k}  \label{eq:ex111}
\end{align}
Writing the expectation with respect to $R$ and $\vx$ we get
\begin{align}
\label{eq:36}
\mathbb E_{R,\xi}[\norms{\tilde{g}_{\lambda_k,R}}]= \frac {\sum_{k=0}^{t-1}  \left(\alpha_*\beta_k-\frac L 2 \beta_k^2\right)\myexpect_{\xi} {\norms{\gt}}}{\sum_{k=0}^{t-1}  \left( \alpha_*\beta_k-\frac L 2 \beta_k^2\right)},
\end{align}
whose numerator is the left side of \eqref{eq:ex111}. Dividing \eqref{eq:ex111} by $\sum_{k=0}^t \left(\alpha_*\beta_k-\frac L 2 \beta_k^2\right)$ and using this in~\eqref{eq:36} we get the result. 
\end{proof}

By substituting $\beta_k = \alpha_*/L$  and $M_k = M$ in \eqref{eqtm22},
\begin{align}
\myexpect_{R,\boldsymbol{\vxi}} [\norm{G_R}^2]
&\le  \frac{C_0+ \frac{c\sigma^2}{2} \sum_{k=0}^{t-1} \frac{\beta_k}{M_k}} {\sum_{k=0}^{t-1} \rnd{ \alpha_* \beta_k - L\beta_k^2/2}} \\
&= \frac{C_0+ \frac{c\sigma^2\alpha_* t}{2L M}} {\frac{\alpha_*^2 t}{2L} } =  \rnd{\frac{2LC_0}{\alpha_*^2 t} + \frac{c\sigma^2}{M\alpha^*}}
\end{align}
The probability distribution for $R$ reduces to a uniform distribution in this case, with the probability of each iteration being $1/t$. This proves Proposition 3.

\section{Derivation of Closed-Form Updates for the GP Model}
The PG-SVI iterations $\vlambda_{k+1} = \min_{\lambda\in \mathcal{S}} \vlambda^T \sqr{ \widehat{\nabla} f(\vlambda_k)} + h(\vlambda) + \frac{1}{\beta_k} \mathbb{D}(\vlambda\|\vlambda_k)$ takes the following form for the GP model, as discussed in Section 6 of the main paper:
\begin{flalign}
(\vm_{k+1},\vV_{k+1}) = \arg\min_{\mathbf{m},\mathbf{V}\succ 0} \,\, 
& (m_n \alpha_{n_k,k} + \half v_{n} \gamma_{n_k,k}) 
 + D_{KL}\sqr{\gauss(\vz|\vm,\vV) || \gauss(\vz|0,\vK) }  &\nonumber\\
&\quad\quad\quad + \frac{1}{\beta_k} D_{KL}\sqr{\gauss(\vz|\vm,\vV) || \gauss(\vz|\vm_k,\vV_k) }. & 
\label{eq:VG_prox}
\end{flalign}
where $n_k$ is the example selected in $k$'th iteration.
We will now show that its solution can be  obtained in closed-form.

\subsection{Full Update of $\vV_{k+1}$}
We first derive the full update of $\vV_{k+1}$.
The KL divergence between two Gaussian distributions is given as follows:
\begin{align}
D_{KL}\sqr{\gauss(\vz|\vm,\vV) || \gauss(\vz|0,\vK) } &= -\half[\log|\vV\vK^{-1}| - \trace(\vV\vK^{-1}) - \vm^T \vK^{-1} \vm + D]
\end{align}
Using this, we expand the last two terms of \eqref{eq:VG_prox} to get the following,
\begin{align}
\label{eq:proximal_objective_function}
&- \half \sqr{\log|\vV\vK^{-1}| - \trace(\vV\vK^{-1}) - \vm^T \vK^{-1} \vm + D} \nonumber\\
& - \half\frac{1}{\beta_k}\sqr{\log|\vV\vK^{-1}| - \trace\{\vV\vV_k^{-1}\} - (\vm -\vm_k)^T \vV_k^{-1} (\vm -\vm_k) + D}  \nonumber \\
&= -\half \left[ \rnd{1+\frac{1}{\beta_k}} \log|\vV| - \trace\{ \vV(\vK^{-1} + \frac{1}{\beta_k} \vV_k^{-1}) \} - \vm^T \vK^{-1}\vm  \nonumber \right. \\
&\quad\quad\quad \left. - \frac{1}{\beta_k} (\vm -\vm_k)^T \vV_k^{-1} (\vm -\vm_k) + \rnd {1+\frac{1}{\beta_k}} \rnd{D - \log|\vK| } \right]  
\end{align}

Taking derivative of \eqref{eq:VG_prox} with respect to $\vV$ at $\vV=\vV_{k+1}$ and setting it to zero, we get the following (here $\vI_n$ is a matrix with all zeros, except the $n$'th diagonal element which is set to 1):
\begin{align}
\Rightarrow\quad& - \rnd{1+ \frac{1}{\beta_k}} \vV_{k+1}^{-1} + \rnd{\vK^{-1} + \frac{1}{\beta_k}\vV_{k}^{-1}} + \gamma_{n_k,k} \vI_{n_k} = 0 \\ 
\Rightarrow\quad& \vV_{k+1}^{-1} = \frac{1}{1+\beta_k} \vV_k^{-1} + \frac{\beta_k}{1+\beta_k} \rnd{ \vK^{-1} + \gamma_{n_k,k}\vI_{n_k} } \\
\Rightarrow\quad& \vV_{k+1}^{-1} = r_k \vV_k^{-1} + (1-r_k) \rnd{ \vK^{-1} + \gamma_{n_k,k}\vI_{n_k} } \label{eq:full_update_V}
\end{align}
which gives us the update of $\vV_{k+1}$ for $r_k := 1/(1+\beta_k)$.

\subsection{Avoiding a full update of $\vV_{k+1}$}
A full update will require storing the matrix $\vV_{k+1}$. Fortunately, we can avoid storing the full matrix and still do an exact update. The key point here is to notice that to compute the stochastic gradient in the next iteration we only need one diagonal element of $\vV_{k+1}$ rather than the whole matrix. Specifically, if we sample $n_{k+1}$'th example at the iteration $k+1$, then we need to compute $v_{n_{k+1},k+1}$ which is the $n_{k+1}$'th diagonal element of $\vV_{k+1}$. This can be done by solving one linear equation, as we show in this section.
Specifically, we show that the following updates can be used to compute $v_{n_{k+1},k+1}$:
\begin{align}
v_{n_{k+1},k+1} &= \kappa_{n_{k+1},n_{k+1}} - \boldsymbol{\kappa}_{n_{k+1}}^T \rnd{\vK + [\diag(\tvgamma_k)]^{-1}}^{-1} \boldsymbol{\kappa}_{n_{k+1}}, \label{eq:updateV11}
\end{align}
where $\tvgamma_{k} = r_k \tvgamma_{k-1} + (1-r_k) \gamma_{n_k,k}\vone_{n_k}$ ($\vone_n$ is a vector of all zeros except its $n$'th entry which is equal to 1).
We start the recursion with $\tvgamma_0 = \epsilon$ where $\epsilon$ is a small positive number. 

We will now show that $\vV_k$ can be reparameterized in terms of a vector $\tvgamma_k$ which contains accumulated weighted sum of the gradient $\gamma_{n_j,j}$, for all $j\le k$. To show this, we recursively substitute the update of  $\vV_{j}$ for $j<k+1$, as shown below (recall that $n_k$ is the example selected at the $k$'th iteration).
The second line is obtained by substituting the full update of $\vV_k$ by using \eqref{eq:full_update_V}. The third line is obtained after a few simplifications. The fourth line is obtained by substituting the update of $\vV_{k-1}$ and a few simplifications.
\begin{align}
& \vV_{k+1}^{-1} = r_k \vV_k^{-1} + (1-r_k) \sqr{ \vK^{-1} + \gamma_{n_k,k}\vI_{n_k} } \\
&= r_k \sqr{r_{k-1} \vV_{k-1}^{-1} + (1-r_{k-1}) \rnd{ \vK^{-1} + \gamma_{n_{k-1},k-1}\vI_{n_{k-1}}} }  + (1-r_k) \sqr{ \vK^{-1} + \gamma_{n_k,k}\vI_{n_k} }\\
&= r_k r_{k-1} \vV_{k-1}^{-1} + (1-r_kr_{k-1}) \vK^{-1} +  \sqr{ r_k(1-r_{k-1}) \gamma_{n_{k-1},k-1} \vI_{n_{k-1}} + (1-r_k)\gamma_{n_k,k}\vI_{n_k} } \nonumber\\
&= r_k r_{k-1} r_{k-2} \vV_{k-2}^{-1} + (1-r_k r_{k-1} r_{k-2}) \vK^{-1}  \nonumber\\
& + \sqr{ r_k r_{k-1} (1-r_{k-2})\gamma_{n_{k-2},k-2}\vI_{n_{k-2}} + r_k(1-r_{k-1}) \gamma_{n_{k-1},k-1}\vI_{n_{k-1}} + (1-r_k) \gamma_{n_{k},k}\vI_{n_{k}}} 
\end{align}
This update expresses $\vV_{k+1}$ in terms of $\vV_{k-2}$, $\vK$, and gradients of the data example selected at $k,k-1,$ and $k-2$.
Continuing in this fashion until $k=0$, we can write the update as follows:
\begin{align}
&\vV_{k+1}^{-1} = t_k \vV_0^{-1} + (1 - t_k) \vK^{-1} + [ r_k r_{k-1} \ldots r_3 r_2 (1-r_1) \gamma_{n_1,1} \vI_{n_1} \nonumber \\
&+ r_k r_{k-1} \ldots r_4 r_3 (1-r_2) \gamma_{n_2,2} \vI_{n_2} + r_k r_{k-1} \ldots r_5 r_4 (1-r_3) \gamma_{n_3,3} \vI_{n_2} + \ldots \nonumber\\
& +  r_k r_{k-1} (1-r_{k-2})\gamma_{n_{k-2},k-2}\vI_{n_{k-2}} + r_k(1-r_{k-1}) \gamma_{n_{k-1},k-1}\vI_{n_{k-1}} + (1-r_k) \gamma_{n_{k},k}\vI_{n_{k}} ] \label{eq:temp345}
\end{align}
where $t_k$ is the product of $r_k,r_{k-1},\ldots, r_0$. We can write the updates more compactly by defining the accumulation of the gradients $\gamma_{n_j,j}$ for all $j\le k$ by a vector $\tvgamma_k$,
\begin{align}
&\vV_{k+1}^{-1} = t_k \vV_0^{-1} + (1 - t_k) \vK^{-1} + \diag(\tvgamma_k) 
\end{align}
The vector $\tvgamma_k$ can be obtained by using a recursion. We illustrate this below, where we have grouped the terms in \eqref{eq:temp345} to show the recursion for $\tvgamma_k$ (here $\vone_{n}$ is a vector with all zero entries except $n$'th entry which is set to 1):
\begin{center}
    \includegraphics{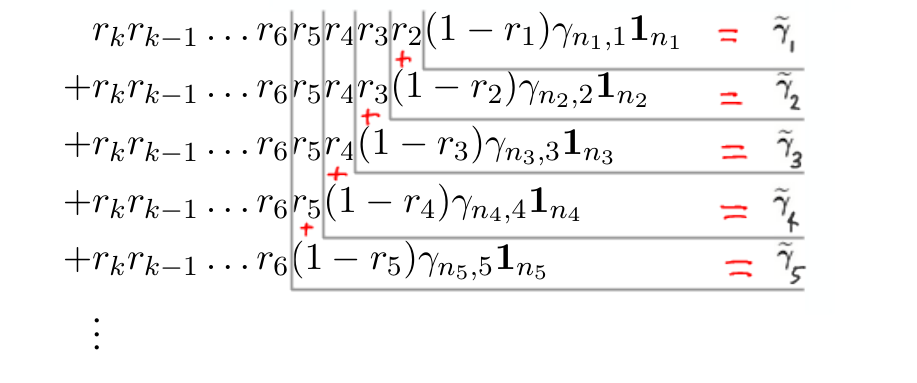}
\end{center}
Therefore, $\tvgamma_k$ can be recursively updated as follows:
\begin{align}
\tvgamma_k = r_k \tvgamma_{k-1} + (1-r_k) \gamma_{n_k,k} \vone_{n_k} 
\end{align}
with an initialization $\tvgamma_0 = \epsilon$ where $\epsilon$ is a small constant to avoid numerical issues.

If we set $\vV_0 = \vK$, then the formula simplifies to the following:
\begin{align}
\vV_{k+1}^{-1} &= \vK^{-1} + \diag(\tvgamma_k) \label{eq:updateVk}
\end{align}
which is completely specified by $\tvgamma_k$, eliminating the need to compute and store $\vV_{k+1}$.

The $n_{k+1}$'th diagonal element can be obtained by using Matrix Inversion Lemma, which gives us the update \eqref{eq:updateV11}.

\subsection{Update of $\vm$}
Taking derivative of \eqref{eq:VG_prox} with respect to $\vm$ at $\vm=\vm_{k+1}$ and setting it to zero, we get the following (here $\vone_{n}$ is a vector with all zero entries except $n$'th entry which is set to 1):
\begin{align}
\Rightarrow\quad& - \vK^{-1}\vm_{k+1} - \frac{1}{\beta_k} \vV_k^{-1} (\vm_{k+1} - \vm_k) -\alpha_{n_k,k} \vone_{n_k}  = 0 \\
\Rightarrow\quad& - \sqr{\vK^{-1} + \frac{1}{\beta_k} \vV_k^{-1}}\vm_{k+1} + \frac{1}{\beta_k} \vV_k^{-1} \vm_k  -\alpha_{n_k,k} \vone_{n_k} = 0 \\
\Rightarrow\quad& \vm_{k+1} = \sqr{\vK^{-1} + \frac{1}{\beta_k} \vV_k^{-1}}^{-1} \sqr{ \frac{1}{\beta_k} \vV_k^{-1} \vm_k - \alpha_{n_k,k} \vone_{n_k}} \\
\Rightarrow\quad& \vm_{k+1} = \sqr{(1-r_k)\vK^{-1} + r_k \vV_k^{-1}}^{-1} \sqr{-(1-r_k)\alpha_{n_k,k} \vone_{n_k} + r_k \vV_k^{-1} \vm_k } 
\end{align}
where the last step is obtained using the fact that $1/\beta_k = r_k/(1-r_k)$.

We simplify as shown below. The second line is obtained by adding and subtracting $(1-r_k)\vK^{-1}\vm_k$ in the square bracket at the right. In the the third line, we take $\vm_{k}$ out. The fourth line is obtained by plugging in the updates of $\vV_k^{-1} = \vK^{-1} +\diag(\tvgamma_k)$. The fifth line is obtained by using Matrix-Inversion lemma, and the sixth line is obtained by taking $\vK^{-1}$ out of the right-most term.
\begin{align}
&\vm_{k+1} =  \sqr{(1-r_k) \vK^{-1} + r_k \vV_k^{-1}}^{-1} \sqr{-(1-r_k)\alpha_{n_k,k} \vone_{n_k} + r_k \vV_k^{-1} \vm_k}\\
&=  \sqr{(1-r_k) \vK^{-1} + r_k \vV_k^{-1}}^{-1} \sqr{(1-r_k) \{- \vK^{-1}\vm_k - \alpha_{n_k,k} \vone_{n_k} \} + \{ (1-r_k)\vK^{-1} + r_k \vV_k^{-1}\} \vm_k} \nonumber\\
&= \vm_k + (1-r_k) \sqr{(1-r_k) \vK^{-1} + r_k \vV_k^{-1}}^{-1} \rnd{-\vK^{-1}\vm_k - \alpha_{n_k,k} \vone_{n_k}}  \\
&= \vm_k - (1-r_k) \sqr{\vK^{-1} + r_k \diag(\tvgamma_{k-1})}^{-1} \rnd{ \vK^{-1}\vm_k + \alpha_{n_k,k} \vone_{n_k}} \\
&= \vm_k - (1-r_k) \sqr{\vK -\vK \rnd{\vK + \diag(r_k \tvgamma_{k-1})^{-1}}^{-1} \vK} \rnd{ \vK^{-1}\vm_k + \alpha_{n_k,k} \vone_{n_k}} \\
&= \vm_k - (1-r_k) \sqr{\vI - \vK \rnd{\vK + \diag(r_k \tvgamma_{k-1})^{-1}}^{-1}} \rnd{ \vm_k + \alpha_{n_k,k}\boldsymbol{\kappa}_{n_k}} \\
&= \vm_k - (1-r_k) (\vI - \vK\vB_k^{-1}) (\vm_k + \alpha_{n_k,k}\boldsymbol{\kappa}_{n_k}) 
\end{align}
where $\vB_k := \vK + [\diag(r_k \tvgamma_{k-1})]^{-1}$.

Since $r_k\tvgamma_{k-1}$ and $\tvgamma_k$ differ only slightly (by the new example gradient $\gamma_{n_k}$, we can instead use the following approximate update:
\begin{align}
\vm_{k+1} &= \vm_k - (1-r_k) (\vI - \vK\vA_k^{-1}) (\vm_k + \alpha_{n_k,k}\boldsymbol{\kappa}_{n_k}) 
\end{align}
where $\vA_k := \vK + [\diag(\tvgamma_{k})]^{-1}$.

\section{Closed-Form Updates for GLMs}
We rewrite the lower bound as
\begin{align}
-\elbofinal(\vm,\vV) := \underbrace{\sum_{n=1}^N f_n(\tm_n,\tv_n)}_{f(\boldsymbol{m},\boldsymbol{V})} + \underbrace{\dkls{}{\gauss(\vz|\vm,\vV)}{\gauss(\vz|0,\vI)}}_{h(\vm,\vV)} \label{eq:glm_lb_1}
\end{align}
where $f_n(\tm_n,\tv_{n}):= -\mathbb{E}_{q}[\log p(y_n|\vx_n^T\vz)]$ with $\tm_n := \vx_n^T$ and $\tv_n := \vx_n^T\vV\vx_n$. We can compute a stochastic approximation to the gradient of $f$ by randomly selecting an example $n_k$ (choosing $M=1$) and using a Monte Carlo gradient approximation to the gradient of $f_{n_k}$.
Similar to GP, we define the following as our gradients of function $f_n$:
\begin{align}
\alpha_{n_k,k} := N \nabla_{\tm_{n_k}} f_{n_k} (\tm_{n_k},\tv_{n_k}), \quad
\gamma_{n_k,k} := 2N \nabla_{\tv_{n_k}} f_{n_k} (\tm_{n_k},\tv_{n_k})
\end{align}
The PG-SVI iteration can be written as follows:
\begin{flalign}
(\vm_{k+1},\vV_{k+1}) = \arg\min_{\mathbf{m},\mathbf{V}\succ 0} \,\, 
& (\tm_n \alpha_{n_k,k} + \half \tv_{n} \gamma_{n_k,k}) 
 + D_{KL}\sqr{\gauss(\vz|\vm,\vV) || \gauss(\vz|0,\vI) }  &\nonumber\\
&\quad\quad\quad + \frac{1}{\beta_k} D_{KL}\sqr{\gauss(\vz|\vm,\vV) || \gauss(\vz|\vm_k,\vV_k) }. & 
\end{flalign}
Using a similar derivation to the GP model, we can show that the following updates will give us the solution: 
\begin{align}
&\tvgamma_k = r_k \tvgamma_{k-1} + (1-r_k) \gamma_{n_k,k} \vone_{n_k} , \nonumber \\
&\tvm_{k+1} = \tvm_k - (1-r_k) (\vI - \vK\vA_k^{-1}) (\vm_k + \alpha_{n_k,k}\boldsymbol{\kappa}_{n_k}) , \nonumber\\
&\tv_{n_{k+1}, k+1} = \kappa_{n_{k+1},n_{k+1}} - \boldsymbol{\kappa}_{n_{k+1}}^T \vA_k^{-1} \boldsymbol{\kappa}_{n_{k+1}},  
\end{align}
where $\vK = \vX\vX^T$ and $\tvm_k := \vX^T\vm$.

\section{Description of the Dataset for Binary GP Classification}
\begin{center}
\begin{tabular}{|l|r|r|r|}
\hline
 & Sonar & Ionosphere & USPS \\
\hline
\# of data points & 208 & 351 & 1,781 \\
\# of features & 60 & 34 & 256 \\
\# of training data points & 165 & 280 & 884 \\
\hline
\end{tabular}
\end{center}

\section{Description of Algorithms for Binary GP Classification}
We give implementation details of all the algorithms used for binary GP- classification experiment.
For all methods, we compute a stochastic estimate of the gradient by using a mini-batch size of 5, 5, and 20 for the three datasets: Sonar, Ionosphere, and USPS-3vs5 respectively.
Similarly, the number of MC samples used are 2000, 500, and 2000.

For GD, SGD, and all the adaptive methods, $\vlambda := \{\vm,\vL\}$ where $\vL$ is the Cholesky factor of $\vV$.
The algorithmic parameters of these methods is given in Table \ref{tab:SGD}. Below, we give details of their updates.

For the GD method, we use the following update:
\begin{equation}
    \vlambda_{k+1} = \vlambda_k + \alpha \nabla \elbofinal(\vlambda_k), 
\end{equation}
where $\alpha$ is a fixed step-size.

For the SGD method, we use a stochastic gradient, instead of the exact gradient: 
\begin{equation}
    \vlambda_{k+1} = \vlambda_k- \alpha_k \vg_k, 
\end{equation}
where $\alpha_k = (k+1)^{-\kappa}$ is the step-size and $\vg_k := -\widehat{\nabla} \elbofinal(\vlambda_k)$.

We use the following updates for ADAGRAD:
\begin{align}
\vs_k &= \vs_{k-1} + \rnd{\vg_k \odot \vg_k},  \\
\vlambda_{k+1} &= \vlambda_k - \alpha_0 \sqr{\frac{1}{\sqrt{\vs_k + \epsilon}}} \odot \vg_k .
\end{align}
where $\alpha_0$ is a fixed step-size and $\epsilon$ is a small constant used to avoid numerical errors.

We use the following update for RMSprop:
\begin{align}
\vs_k &= \rho \vs_{k-1} + (1-\rho) \rnd{\vg_k \odot \vg_k}, \\
\vlambda_{k+1} &= \vlambda_k - \alpha_0 \sqr{\frac{1}{\sqrt{\vs_k + \epsilon}}} \odot \vg_k,
\end{align}
where $\alpha_0$ is a fixed step-size and $\rho$ is the decay factor.

We use the following updates for ADADELTA:
\begin{align}
\vs_k &= \rho \vs_{k-1} + (1-\rho) \rnd{\vg_k \odot \vg_k}, \\
\vlambda_{k+1} &= \vlambda_k - \vg_k^{AD}, \quad \textrm{ where } \vg_k^{AD} = \alpha_0 \rnd{\frac{\sqrt{\vdelta_k + \epsilon}}{\sqrt{\vs_k + \epsilon}}} \odot \vg_k, \\ 
\vdelta_{k+1} &= \rho \vdelta_k + (1-\rho) \rnd{\vg_k^{AD} \odot \vg_k^{AD}}.
\end{align}
where again $\alpha_0$ is a fixed step-size, and $\rho$ is the decay factor.

Finally, the updates for ADAM are shown below:
\begin{align}
\vmu_k &= \rho_\mu \vmu_{k-1} + (1-\rho_\mu) \vg_k, \\
\vs_k &= \rho_s \vs_{k-1} + (1-\rho_s) \rnd{\vg_k \odot \vg_k}, \\
\vg_{s,k} &= \sqrt{\frac{\vs_k}{1-\rho_s^k}}, \\ 
\vlambda_{k+1} &= \vlambda_k - \alpha_0 \sqr{\frac{1}{\vg_{s,k} + \epsilon}} \odot \sqr{\frac{\vmu_k}{1-\rho_\mu^k}}. 
\end{align}
where $\alpha_0$ is a fixed step-size and $\rho_\mu, \rho_s$ are decay factors.

\begin{table}[h]
\caption{Algorithmic parameters for Binary GP classification experiment (Figure 2 in the main paper). $N$ is the number of training examples.} 
\center
\begin{tabular}{|l|r|r|r|}
\hline
Parameter & Sonar & Ionosphere & USPS \\
\hline
\multicolumn{4}{c}{SGD}\\
\hline
$\kappa$  & 0.8  & 0.51 & 0.6  \\
$\alpha_0 \times N$ & 1200 &  25 & 800 \\
\hline
\multicolumn{4}{c}{ADAGRAD}\\
\hline
$\alpha_0$ & 4.5 & 4 & 8 \\
\hline
\multicolumn{4}{c}{RMSprop}\\
\hline
$\alpha_0$ & 0.1 & 0.04 & 0.1 \\
$\rho$ & 0.9 & 0.9999 & 0.9 \\
\hline
\multicolumn{4}{c}{ADADELTA}\\
\hline
$\alpha_0$ & 1.0 & 0.1 & 1.0 \\
$1 - \rho$  & $5 \times 10^{-10} $ & $10^{-11}$ & $10^{-12}$ \\
\hline
\multicolumn{4}{c}{ADAM}\\
\hline
$\alpha_0$ & 0.04 & 0.25 & 2.5 \\
$\rho_{\mu}$  & 0.9 & 0.9 & 0.9  \\
$\rho_{s}$  & 0.999 & 0.999 & 0.999  \\
\hline
\multicolumn{4}{c}{PG-SVI}\\
\hline
$\beta_k \times N$ & 0.2 & 2.0 & 2.5 \\
\hline
\end{tabular}
\label{tab:SGD}
\end{table}

\end{appendix}

\end{document}